\documentclass[accepted]{uai2025} %

\usepackage[american]{babel}

\usepackage{natbib} %
\usepackage{mathtools} %
\usepackage{booktabs} %
\usepackage{tikz} %

\title{The Causal Information Bottleneck and Optimal Causal Variable Abstractions}

\author[1]{\href{mailto:<f.simoes@uu.nl>?Subject=Your UAI 2025 paper}{Francisco~N.~F.~Q.~Simoes}{}}
\author[1]{Mehdi~Dastani}
\author[1]{Thijs~van~Ommen}
\affil[1]{%
  Department~of~Information~and~Computing~Sciences\\
  Utrecht~University\\
  The~Netherlands
}

\newcommand{\diagramscale}{0.7}

\usepackage{amsmath,amsfonts,bm}

\def\eqref#1{equation~\ref{#1}}

\def\1{\bm{1}}

\DeclareMathAlphabet{\mathsfit}{\encodingdefault}{\sfdefault}{m}{sl}
\SetMathAlphabet{\mathsfit}{bold}{\encodingdefault}{\sfdefault}{bx}{n}

\newcommand{\E}{\mathbb{E}}

\newcommand{\R}{\mathbb{R}}

\DeclareMathOperator*{\argmin}{arg\,min}

\usepackage{hyperref}
\usepackage{url}

\usepackage{amsthm} %
\usepackage{amssymb, mathtools}
\usepackage{enumitem, cleveref, nameref} %
\usepackage{cancel}
\usepackage{nicefrac}
\usepackage{stmaryrd} %
\usepackage{mathdots} %
\usepackage{subcaption} %
\usepackage{tikz}
\usetikzlibrary{calc}

\usepackage{soul} %
\soulregister\cite7 %
\soulregister\Citet7
\soulregister\Citep7
\soulregister\citet7
\soulregister\citep7
\soulregister\ref7
\soulregister\cref7
\soulregister\Cref7
\soulregister\addref7
\soulregister\footnotemark7
\soulregister\footnotetext7

\newcommand{\addref}{\textcolor{red}{[ADD REF]}}

\usepackage{booktabs}
\usepackage[bottom]{footmisc} %
\usepackage{xargs} %

\let\todo\undefined %
\usepackage[colorinlistoftodos,prependcaption,textsize=tiny]{todonotes}

\DeclareMathOperator{\cC}{\mathfrak{C}} 
 \newcommand{\dop}{\mathit{do}}

\newcommand{\Pa}{\mathrm{Pa}}

\newcommand{\Ch}{\mathrm{Ch}}

\newcommand{\abs}[1]{\lvert #1 \rvert} 

\newcommand{\C}{\mathfrak{C}}

\newcommand{\LIB}{\mathcal{L}^{\beta}_{\mathrm{IB}}}
\newcommand{\LCIB}{\mathcal{L}^{\beta}_{\mathrm{CIB}}}
\newcommand{\wLCIB}{\mathcal{L}^{\gamma}_{\mathrm{wCIB}}}
\newcommand{\LCIBgamma}{\mathcal{L}^{\gamma}_{\mathrm{CIB}}}
\newcommand{\suchthat}{\ \text{ s.t. }\ }

\newcommand{\VI}{\mathrm{VI}}

\newcommand{\quotient}[2]{{\raisebox{.1em}{$#1$}\left/\raisebox{-.1em}{$#2$}\right.}}
\newcommand{\supp}{\mathrm{supp}}

\definecolor{gold}{rgb}{1.0, 0.84, 0.0}

\newcommandx{\warning}[2][1=]{\todo[linecolor=red,backgroundcolor=red!25,bordercolor=red,#1]{#2}}
\newcommandx{\note}[2][1=]{\todo[linecolor=blue,backgroundcolor=blue!25,bordercolor=blue,#1]{#2}}
\newcommandx{\question}[2][1=]{\todo[linecolor=purple,backgroundcolor=purple!30,bordercolor=purple,#1]{#2}}
\newcommandx{\Q}[2][1=]{\todo[linecolor=red!125,backgroundcolor=red!50,bordercolor=red,#1]{#2}}
\newcommandx{\original}[2][1=]{\todo[linecolor=green,backgroundcolor=green!25,bordercolor=green,#1]{OG!: #2}}
\newcommandx{\optional}[2][1=]{\todo[linecolor=yellow,backgroundcolor=yellow!25,bordercolor=yellow,#1]{#2}}

\newtheorem{theorem}{Theorem}

\newtheorem{proposition}[theorem]{Proposition}

\newtheorem{definition}[theorem]{Definition}

\newtheorem{assumption}[theorem]{Assumption}
\theoremstyle{remark}
\newtheorem{remark}[theorem]{Remark}

\newtheorem*{notation*}{Notation}

\begin{document}
\maketitle

\begin{abstract}
  To effectively study complex causal systems, it is often useful to construct abstractions of parts of the system by discarding irrelevant details while preserving key features.
  The Information Bottleneck (IB) method is a widely used approach to construct variable abstractions by compressing random variables while retaining predictive power over a target variable.
  Traditional methods like IB are purely statistical and ignore underlying causal structures, making them ill-suited for causal tasks.
  We propose the Causal Information Bottleneck (CIB), a causal extension of the IB, which compresses a set of chosen variables while maintaining causal control over a target variable.
  This method produces abstractions of (sets of) variables which are causally interpretable, give us insight about the interactions between the abstracted variables and the target variable, and can be used when reasoning about interventions.
  We present experimental results demonstrating that the learned abstractions accurately capture causal relations as intended.
\end{abstract}

\section{Introduction}
\label{sec:introd}

Natural systems typically consist of a vast number of components and interactions, making them complex and challenging to study.
When investigating a specific scientific question, which is often of a causal nature, it is frequently possible to disregard many of these details, as they have a negligible impact on the outcome.
These details can then be abstracted away.
A classic example \citep{rubenstein2017causal,chalupka2017overview} is the relationship between particle velocities, temperature, and pressure.
To control the pressure on the walls of a room, it would presumably be necessary to consider how to manipulate the velocities of the approximately $10^{23}$ particles in the room.
However, accounting for the velocity of each individual particle is not required to achieve this.
Instead, considering solely interventions on the high-level variable of temperature is sufficient.
Other examples of such high-level features include large-scale weather phenomena like El Nino, which abstract away low-level sea surface temperature details irrelevant to resulting wind patterns \citep{chalupka2016unsupervised}, and visual features of images, which omit information not relevant to triggering neural spikes \citep{chalupka2017thesis}.
In general, having such an abstraction variable available can aid reasoning about interventions.

Methods that disregard the causal structure of a system when constructing abstractions may yield results that are uninformative or even misleading, particularly when the objective is to manipulate the system or gain causal insights.
The following running example will serve as a useful illustration of the potential drawbacks of neglecting causal considerations when learning abstractions.
Consider a mouse gene with four positions $s_{1}, s_{2}, s_{3}, s_{4}$ under study where nucleotides
may be mutated, corresponding to the binary variables $X_{i},\ i=1,\ldots,4$, which indicate whether there is a mutation at position $s_{i}$.
These mutations can interact in a complex manner with respect to a phenotype of interest $Y$, say the body mass of the mouse.
This type of complex interaction is known as epistasis \citep{phillips2008epistasis}.
One could create an abstraction $T$ of $X_{1}, X_{2}, X_{3}, X_{4}$ that would be blind to differences between mutation configurations $(X_{1}, X_{2}, X_{3}, X_{4})$ that do not provide information about $Y$.
With the current description, it may seem that there is no need for considering causality, and that one could simply use a purely statistical method to learn a good ``epistasis gauger'' $T$.
This is, however, not the case, since the $X_{i}$ and $T$ are typically confounded by the population structure, that is, ``any form of relatedness in the sample,
including ancestry differences or cryptic relatedness'' \citep{sul2018population}.
Based on the discussion in \citet{sul2018population}, we consider the population structure variable $S$ encoding the strain of the mice (laboratory vs wild-derived strains), which is associated with both distinct $X$ and body masses $Y$ (high in laboratory strains and low in wild-derived strains).
A possible resulting causal graph can be seen in \Cref{fig:mutations-example}.
As a result, mutations that are more prevalent in mice from the wild-derived strain will exhibit a strong correlation with low body mass, even if there is no underlying causal relationship between them.
This is an example of how population structure can work as a confounder.

The construction of abstractions, often in the form of representations, has been a prominent area of research in \emph{non-causal} machine learning for many years, with both theoretical and applied contributions.
A ``good'' representation/abstraction, should keep only the relevant information of the abstracted variable $X$ \citep{bishop2006pattern,goodfellow2016deep}.
It is essential to recognize that the determination of which information is spurious critically depends on the specific task at hand.
Building an abstraction is therefore a balancing act between keeping enough information for the task and avoiding including unnecessary details.
This is the insight which \citet{tishby2000information} build on.
The authors formalize the learning of the optimal abstraction $T$ for an input variable $X$ and a chosen target variable $Y$ as a minimization problem.
This problem involves finding a balance between compressing $X$ as much as possible and maintaining as much information about $Y$ as possible.
The functional proposed by \citet{tishby2000information} to be minimized is called the \emph{Information Bottleneck (IB) Lagrangian}, and it has inspired a significant body of work in representation learning (see for example \citet{alemi2016deep,kolchinsky2018caveats,artemy2019nonlinear,tishby2015bottleneck,achille2018information}).
The IB method, however, does not account for causality.
The learned abstractions cannot be used to reason causally about the system.
In particular, when $X$ and $Y$ are heavily confounded, the IB method will create an abstraction $T$ that preserves the information $X$ has about $Y$, but a significant portion of this information is spurious from a causal perspective.
This not only leads to sub-optimal compression but can also result in misleading conclusions, where the values of $T$ are mistakenly interpreted as corresponding to meaningful interventions.

  In this paper, we present a new method for learning abstractions of a set of input variables $X$ which retain the causal information between $X$ and a specified target variable $Y$.
  Namely, we introduce a causal version of the IB Lagrangian designed to attain its minima when the abstraction $T$ of $X$ compresses $X$ as much as possible while maintaining as much \emph{causal} control over $Y$ as desired, with the trade-off between these properties governed by a parameter $\beta$.
We derive the \emph{Causal Information Bottleneck (CIB) Lagrangian} by first establishing an axiomatic description of what constitutes an \emph{optimal causal variable abstraction}\footnotemark.
\footnotetext{We use the term ``variable abstraction'' instead of simply ``abstraction'' to make it clear that our focus is on abstractions of a (set of) variable(s), not of an entire model, for which the term ``causal abstraction'' is often used (see \Cref{sec:related-work})}
As in the original Information Bottleneck paper \citep{tishby2000information}, we take all variables to be discrete, and the abstractions are learned using solely the joint distribution.
The CIB method does not require access to the full DAG; it only requires that post-intervention distributions are either provided or identifiable.
In our experiments, the causal effect of $X$ on $Y$ is taken to be identifiable by the use of the backdoor criterion.
Our experiments confirm that optimal causal variable abstractions $T$ reveal the form of the causal interactions among the abstracted variables $X_{i}$.
Thus, and because $T$ preserves the causal control that $X$ has over $Y$, one can use $T$ instead of $X$ to reason about what interventions to perform.
$T$ can then be seen as an auxiliary variable which can give us insights about the causal relationship between $X$ and $Y$, and how to intervene on $X$.

The contributions of this paper can be stated as follows.
(i) We propose an axiomatic definition of \emph{optimal causal variable abstraction} (OCVA), which naturally extends previous non-causal definitions.
(ii) We derive a causal version of the Information Bottleneck Lagrangian from those axioms, thereby formulating the problem of optimal causal variable abstraction learning as a minimization problem.
(iii) We propose a definition of variable abstraction intervention, and obtain a backdoor criterion formula for variable abstractions, enabling us to compute an abstractions's post-intervention distribution from observational data.
As a supplementary result, we propose a definition of equivalence for variable abstractions and show that the variation of information between two variables representations is zero when they are equivalent.

All proofs can be found in the Appendix, which also contains various supplements to the main text.
The code repository containing the experiments is available at \url{github.com/francisco-simoes/cib-optimization-psagd}.

\section{Preliminaries}
\label{sec:preliminaries}

\paragraph{Causal Models}
A Structural Causal Model (SCM) provides a representation of a system's causal structure, analogous to a Bayesian network but with a causal interpretation.
An SCM $\C = (\mathbf{V}, \mathbf{N}, S, p_{\mathbf{N}})$ is comprised of endogenous variables $\mathbf{V}$, exogenous variables \textbf{N}, deterministic functions between them $S$, and a distribution over the noise variables $p_{\mathbf{N}}$.
\textbf{V} can consist of both observed and hidden variables.
Each SCM $\cC$ has an underlying DAG $G^{\cC}$, called its causal graph.
Each node in the DAG corresponds to an endogenous variable, while the edges stand for the causal relationships between them.
We denote the parents and children of an endogenous variable $X$ by $\Pa(X)$ and $\Ch(X)$, respectively.
Furthermore, we denote the range of a random variable $X$ by $R_{X}$ and its support by $\supp(X)$.
The value of each endogenous variable is determined by a deterministic function of its parent variables and an independent exogenous variable, which accounts for the system's randomness.
A key feature of SCMs is their capacity to model interventions on a variable $X$, which involve altering the variable's generating process.
This results in a new SCM with its own distribution, reflecting the system's state post-intervention.
The most common type of intervention is an atomic intervention, where a variable $X$ is set to a specific value $x$, effectively severing its connection to its parents and assigning a fixed value instead.
We denote such an intervention by $\dop(X=x)$, the resulting SCM by $\cC^{\dop(X=x)}$, and the post-intervention joint distribution of a set of variables $W$ by $p_{W}^{\dop(X=x)}(w)$ or $p(w\mid \dop(X=x))$.
For more details, see \Cref{sec:struct-caus-models}.

\paragraph{Causal Entropy and Causal Information Gain}
We now introduce two concepts: causal entropy and causal information gain, both of which are fundamental to our method.
For more details, see \Cref{sec:appendix-caus-entr-caus}, or refer to \citet{simoes2023causal}.
The causal entropy $H_{c}(Y\mid \dop(X))$ measures the average uncertainty remaining about the variable \( Y \) after we intervene on the variable \( X \).
This concept is closely related to conditional entropy but adapted for situations where interventions on \( X \), as opposed to conditioning on $X$, are considered.
The causal information gain $I_{c}(Y\mid \dop(X))$ extends the idea of mutual information to the causal domain.
It quantifies the reduction in uncertainty about \( Y \) after intervening on \( X \), offering a measure of the causal control that \( X \) exerts over \( Y \).
Thus, it tells us how much more we know about \( Y \) due to these interventions on \( X \).

\paragraph{The Information Bottleneck Lagrangian}
Let $X$ be a random variable.
By a variable abstraction (or v-abstraction for short) $T$ of $X$, we mean a variable that can only depend on $X$, whether deterministically or stochastically.
This means in particular that $T$ must be independent of $Y$ when conditioning on $X$.
This generalizes the notion of representation used by \citet{achille2018information} for cases with more variables than only $X$ and $Y$.
Furthermore, a v-abstraction $T$ is characterized by its encoder, which codifies how $T$ depends on $X$.
We formalize this as follows:

\begin{definition}[Variable Abstraction and Encoder]
  \label{def:rep}
  A random variable $T$ is a \emph{variable abstraction} (also called \emph{v-abstraction}) of a random variable $X$ if $T$ is a function of $X$ and an independent noise variable.
  The \emph{encoder} of the v-abstraction $T$ is the function $q_{T\mid X}\colon R_{T} \times R_{X} \to [0, 1]$ such that $q_{T\mid X}(t \mid x)$ is the conditional probability $q(t\mid x)$ of $T=t$ given $X=x$.
\end{definition}

In \citet{tishby2000information}, the authors aim at finding a v-abstraction\footnotemark $T$ of $X$ which ``keeps a fixed amount of meaningful information about the relevant signal $Y$ while minimizing the number of bits from the original signal $X$ (maximizing the compression).''
\footnotetext{They call them ``representations.'' We use the term (causal) variable abstraction instead, as ``representation'' is already used in the causality literature for tasks different from ours --- for example, in ``causal representation learning'' (see \Cref{sec:related-work}).}
This is accomplished by introducing the Information Bottleneck Lagrangian $\LIB[q_{T\mid X}] = I(X; T) - \beta I(Y; T)$, where $\beta$ is a non-negative parameter which manages the trade-off between compression (as measured by $I(X;T)$) and sufficiency (as measured by $I(Y;T)$).
The problem of finding such an abstraction becomes then the problem of finding an encoder $q_{T\mid X}$ which minimizes the $\LIB$.
They minimize this Lagrangian by adapting the Blahut-Arimoto algorithm from rate distortion theory \citep{blahut1972computation} to their case, resulting in a coordinate-descent optimization algorithm.
This adaptation hinges on viewing mutual information as a KL divergence.

The IB method proposed by \citet{tishby2000information} effectively solves the following minimization problem:
\begin{equation}
  \label{eq:tishby_min_problem}
  \argmin_{q_{T\mid X}} I(X; T)
  \suchthat
  \begin{cases}
    \forall x\in R_{X}, q_{T\mid X=x} \in \Delta^{\abs{R_{T}}-1}\\
    I(Y; T) = D
  \end{cases},
\end{equation}
where $\Delta^{\abs{R_{T}}-1}$ is the probability simplex, $q_{T\mid X}\in \mathbb{R}^{\abs{R_{T}}\cdot \abs{R_{X}}}$, and $D$ belongs to the sufficiency values compatible with the chosen $\beta$.
Notice that in particular $q_{T\mid X}$ is constrained to $\Delta \coloneqq \bigtimes_{x \in R_{X}} \Delta^{\abs{R_{T}} - 1}$ (see \Cref{sec:learning_algo}).

\section{Optimal Causal Variable Abstractions}
\label{sec:optimal-causal-rep}

For the remainder of the paper, let $X \subseteq \mathbf{V}$ be a set of endogenous variables of an SCM $\C = (\mathbf{V}, \mathbf{N}, S, p_{\mathbf{N}})$, $T$ be a v-abstraction of $X$ with encoder $q_{T\mid X}$, and $t$ be an element of $R_{T}$.

In a natural extension of the description of optimal v-abstraction in \Cref{sec:preliminaries} to the causal context, our problem can be described as finding v-abstractions $T$ of $X$ which retain a chosen amount $D$ of \emph{causal} information about the relevant signal $Y$ while minimizing the information that $T$ preserves about $X$.
We propose an axiomatic characterization of optimal causal v-abstraction to formally capture this description using information-theoretical quantities.
This can also be seen as a causal variant of the characterization of optimal representation from \citet{achille2018information}.
Since we use $I_{c}(Y\mid \dop(T))$ to measure the causal information that $T$ has about $Y$, and $I(X;T)$ is the information that $T$ keeps about $X$, the result is the following:

\begin{definition}[Optimal Causal Variable Abstraction]
  \label{def:optimal_causal_rep}
  A \emph{optimal causal variable abstraction} (OCVA) of $X$  at sufficiency $D$ is a v-abstraction $T$ of $X$ such that:
  \begin{enumerate}[label=(C\arabic*)]
    \item $T$ is \emph{interventionally $D$-sufficient} for the task $Y$, \emph{i.e.}, $I_{c}(Y\mid \dop(T)) = D$. \label{item:caus_max_suff}
    \item $I(X;T)$ is minimal among the variables $T$ satisfying \ref{item:caus_max_suff}.
  \end{enumerate}
\end{definition}

We can then formulate the problem of finding an OCVA as the following minimization problem:
\begin{equation}
  \label{eq:our_min_problem}
  \argmin_{q_{T\mid X}} I(X; T)
  \suchthat
  \begin{cases}
    \forall x\in R_{X}, q_{T\mid X=x} \in \Delta^{\abs{R_{T}}-1}\\
    I_{c}(Y\mid \dop(T)) = D
  \end{cases},
\end{equation}
where $q_{T\mid X} \in \mathbb{R}^{\abs{R_{T}}\cdot \abs{R_{X}}}$.
Recall that the IB Lagrangian can be seen as arising from the minimization problem in \Cref{eq:tishby_min_problem}.
Likewise, the CIB Lagrangian will be introduced to solve the minimization problem in \Cref{eq:our_min_problem}.

\section{The Causal Information Bottleneck}
\label{sec:caus-inform-bottl}

We can find the solution(s) to \Cref{eq:our_min_problem} within the minimizers of a Lagrangian.
Specifically, we can minimize (for some chosen $X$, $Y$ and $R_{T}$) the \emph{causal information bottleneck Lagrangian} $\LCIB$ defined as follows:
\begin{definition}[Causal Information Bottleneck]
The \emph{causal information bottleneck (CIB) Lagrangian} with trade-off parameter $\beta \ge 0$ is the function $\LCIB\colon \R^{\abs{R_{T}}} \times \R^{\abs{R_{X}}} \to \R$ given by
  \begin{equation}
  \label{eq:cib_lag}
    \LCIB[q_{T\mid X}] \coloneqq I(X;T) - \beta I_{c}(Y\mid\dop(T)).
  \end{equation}
\end{definition}
The trade-off parameter $\beta$ is the Lagrange multiplier for the interventional sufficiency constraint, and is taken to be fixed.
Notice that the constraint $q_{T\mid X} \in \Delta$ still needs to be enforced explicitly.
As for the IB Lagrangian, fixing $\beta$ restricts the values $D$ that $I_{c}(Y \mid \dop(T))$ can take when minimizing $\LCIB$ \citep{kolchinsky2018caveats}.
In general, larger $\beta$ favors interventional sufficiency over compression, resulting in larger values of $D$ and $I(X;T)$.
We assume $\beta$ to be a non-negative real number in view of its interpretation as a trade-off parameter.

As discussed in \Cref{sec:learning_algo}, developing a method to minimize \Cref{eq:cib_lag} analogous to that of \citet{tishby2000information} is challenging, if not infeasible.
Instead, in our experiments we employ constrained optimization iterative local search algorithms based on gradient descent to find the minima of $\LCIB$ while adhering to the probability simplices constraint $q_{T\mid X} \in \Delta$.

\section{Interventions on Abstractions}
\label{sec:interv-abst}

As discussed in \Cref{sec:introd}, there is often interest in creating abstractions that can be intervened on.
An intervention on a v-abstraction $T$ of $X$ must correspond to interventions on the low-level variables $X$.
Therefore, an intervention on $T$ will induce a distribution over the possible interventions on $X$, which we denote by $p^{\star}(x\mid t)$ and will refer to as the ``intervention decoder''.
We require that $p^{\star}(x\mid t)$ be compatible with the encoder\footnotemark $q(t\mid x)$, in the sense that both must agree with a common joint distribution over $T$ and $X$.
\footnotetext{See \Cref{rem:int-decoder-def} (\Cref{sec:appendix-int-on-abst}) for an extended discussion about this requirement.}

\begin{definition}[Intervention Decoder]
  The \emph{intervention decoder} $p^{\star}(x\mid t)$ for the v-abstraction $T$ of $X$ is computed from the encoder $q_{T\mid X}$ using the Bayes rule with a chosen prior $p^{\star}(x)$, that is,
  \begin{equation}
    \label{eq:int-decoder}
    p^{\star}(x\mid t) \vcentcolon= \frac{q(t\mid x) p^{\star}(x)}{\sum_{\dot{x}} q(t\mid \dot{x}) p^{\star}(\dot{x})}
  \end{equation}
\end{definition}
Notice that a choice of prior over the possible atomic distributions on $X$ still needs to be made.
In practice, we will make the choice that $p^{\star}(x)$ be uniform, so that $p^{\star}(x\mid t) = \frac{q(t\mid x)}{\sum_{\dot{x}} q(t\mid \dot{x})}$.

In order to compute the effect of intervening on $T$ on the SCM variables $\mathbf{V}$, we compute the effect of the atomic interventions $\dop(X=x)$, and weight them with the likelihood of that intervention using the intervention decoder.
\begin{definition}[Variable Abstraction Intervention]
  \label{def:abst-intervention}
  The \emph{v-abstraction intervention distribution} $p_{\mathbf{V}}^{\dop(T=t)}$ is the weighted average of the atomic intervention distributions $p_{\mathbf{V}}^{\dop(X=x)}$ over $x\in R_{X}$, where the weights are given by the intervention decoder $p^{\star}(x\mid t)$.
  That is,
  \begin{equation}
    \label{eq:abst-intervention}
    p(\mathbf{v} \mid \dop(T=t)) \vcentcolon= \sum_{x} p^{\star}(x\mid t) p_{\mathbf{V}}^{\dop(X=x)}(\mathbf{v}).
  \end{equation}
\end{definition}

\section{Backdoor Adjustment for Abstractions}
\label{sec:backdoor-for-abst}
In order to learn the OCVA, the learning algorithm will need to estimate $\LCIB = I(X;T) - \beta I_{c}(Y\mid \dop(T))$ at each iteration step, making use of the encoder $q_{T\mid X}$ at that iteration and the joint $p_{\mathbf{V}}$.
The compression term $I(X;T)$ can be directly computed using $p_{X}$ and the encoder, while the interventional sufficiency term $I_{c}(Y\mid \dop(T))$ demands the computation of $p(y\mid \dop(T=t))$.
\Cref{eq:pydot} allows us to write $p(y\mid \dop(T=t))$ in terms of the encoder $q(t\mid x)$ and the intervention distributions $p(y\mid \dop(X=x))$.
Hence, $p(y\mid \dop(T=t))$ is identifiable (\emph{i.e.}, computable from the joint $p_{\mathbf{V}}$) if $p(y\mid \dop(X=x))$ also is.
Satisfaction of the \emph{backdoor criterion} is one of the most common ways to have identifiability.
In case a set $Z\subseteq \mathbf{V}$ exists satisfying the backdoor criterion relative to $(X, Y)$, then $p(y\mid\dop(X=x))$ is identifiable, and is given by the backdoor adjustment formula \citep{pearl2009causality}.
From the definition of v-abstraction intervention and the backdoor adjustment formula, one can derive a backdoor adjustment formula for v-abstractions.
\begin{proposition}[Backdoor Adjustment Formula for Abstractions]
  \label{prop:backdoor-for-abs}
  Let $T$ be a v-abstraction of $X$ with encoder $q(t\mid x)$, and $Z$ be a set of variables of $\mathfrak{C}$ satisfying the backdoor criterion relative to $(X, Y)$ in $\mathfrak{C}$.
  Then the v-abstraction intervention distribution for $Y$ is identifiable and is given by:
  \begin{equation}
    \label{eq:backdoor-for-abstractions}
    \begin{split}
      p(y \mid& \dop(T=t)) = \sum_{z} p(z) \sum_{x} p(y \mid x,z) p^{\star}(x\mid t) \\
            &= \sum_{z} p(z) \sum_{x} p(y \mid x,z) \frac{q(t\mid x) p^{\star}(x)}{\sum_{\dot{x}}q(t\mid \dot{x})p^{\star}(\dot{x})},
    \end{split}
  \end{equation}
  where $p^{\star}(x\mid t)$ is the intervention decoder for $T$, and $p^{\star}(x)$ is the prior over interventions on $X$.
\end{proposition}

\begin{remark}[Interpretation of the Backdoor Adjustment Formula for Abstractions]
  Since $Z$ meets the backdoor criterion for $(X, Y)$, the probability $p(y\mid x, z)$ can be seen as the probability of observing $Y=y$ in the subpopulation $Z=z$ given that one does $\dop(X=x)$.
  Furthermore, $p^{\star}(x\mid t)$ is the probability that $X$ is set to $x$, given that $T$ is set to $t$.
  Thus, the sum over $x$ can be seen as the average effect of $X$ on $Y$ in the subpopulation $Z=z$ and given that $T$ is set to $t$.
  Hence the post-intervention probability $p(y\mid \dop(t))$ is the average over all subpopulations $Z=z$ of the average effects of $X$ on $Y$, given that $T=t$.
\end{remark}

\Cref{prop:backdoor-for-abs} allows us to express the interventional sufficiency term of the CIB in terms of the encoder and the joint $p_{XYZ}$ as long as there is a set $Z$ (of observed variables) which satisfies the backdoor criterion relative to $(X,Y)$.
The expression for $I_{c}(Y\mid \dop(T))$, and thus for the CIB, will have the same form in all such cases.
From now on, we will assume that such a set $Z$ is observed and thus can be conditioned on.
\begin{assumption}
  \label{assump:backdoor}
  A set of random variables $Z$ satisfying the backdoor criterion \citep{pearl2009causality} relative to $(X, Y)$ in the SCM $\mathfrak{C}$ is observed.
\end{assumption}

\Cref{assump:backdoor} is introduced for the sake of convenience and simplifying the algorithm.
If violated, $p(y\mid \dop(t))$ may still be identifiable.
In that case, one can make use of do-calculus \citep{peters2017elements} to obtain an expression for $p(y\mid \dop(x))$, and thus also for $p(y\mid \dop(t))$, which will hold for that specific causal graph.

\section{Comparing Variable Abstractions}
\label{sec:comp-abstr}

After learning a v-abstraction \( T_{1} \), we may want to compare it with another v-abstraction \( T_{2} \), which could either be one learned earlier or one considered as the ground truth.
Simple equality of their encoders is not an appropriate criterion for this comparison.
Two v-abstractions might have different encoders and still be ``equivalent'' in the sense that they coincide when the values of the v-abstraction are relabeled.
This is especially apparent in the case of completely deterministic v-abstractions.
For example, consider two binary v-abstractions \( T_{1} \) and \( T_{2} \) of a low-level variable \( X \) with range \( R_{X} = \{0, 1, 2\} \), defined by deterministic functions \( \phi_{T_{1}} \) and \( \phi_{T_{2}} \).
Suppose \( \phi_{T_{1}} \) maps 0 and 1 to 0, and 2 to 1, while \( \phi_{T_{2}} \) maps 0 and 1 to 1, and 2 to 0.
Intuitively, $T_{1}$ and $T_{2}$ represent the same v-abstraction because relabeling the values of $T_{1}$ (\emph{i.e.}, swapping 0 and 1) yields a v-abstraction that produces the same values as $T_{2}$ for the same low-level values of $X$. Formally, and extending to the non-deterministic case, this is to say that equivalence arises whenever the conditional distributions are identical up to a bijection \( \sigma \) of the values of the \( T_{i} \).
This leads to the following definition:
\begin{definition}[Equivalent Variable Abstractions]
  Two v-abstractions $T_{1}$ and $T_{2}$ of $X$ are \emph{equivalent} if there is a bijection $\sigma\colon \supp(T_{1}) \to \supp(T_{2})$ such that $\forall t_{1} \in \supp(T_{1}), x \in \supp(X), \  q_{T_{1}\mid X}(t_{1} \mid x) = q_{T_{2}\mid X}(\sigma(t_{1}) \mid x)$, where $q_{T_{1}\mid X}$, $q_{T_{2}\mid X}$ are the encoders for $T_{1}$ and $T_{2}$.
  We then write $T_{1} \cong T_{2}$.
\end{definition}

One can show that $\cong$ is an equivalence relation (see \Cref{prop:cong-is-equiv-rel}).
We call \emph{abstraction class} an equivalence class of $\cong$, that is, the elements of $\quotient{\Delta}{\cong}$.
In practice, it is unlikely that two v-abstractions will be exactly equivalent.
Thus, it will be useful to have a measure that quantifies the dissimilarity between two v-abstractions, indicating how far they are from being equivalent.
This measure should be minimized when the v-abstractions are equivalent.
We do not present such a metric for the general case (\emph{i.e.} for arbitrary $T_{1}$ and $T_{2}$):
in our experiments, we will want to compare learned encoders with a deterministic encoder corresponding to the ground truth $\underline{T}$ for the case $\beta\rightarrow +\infty$.
Hence, a metric that captures equivalence between a v-abstraction $T$ and a deterministic v-abstraction $\underline{T}$ will suffice.
The variation of information \citep{meilua2003comparing,meilua2007comparing} $\VI(T_{1}; T_{2}) \coloneqq H(T_{1}\mid T_{2}) + H(T_{2} \mid T_{1})$ is such a metric: $\VI(T, \underline{T}) = 0$ exactly when $T$ and $\underline{T}$ are equivalent (see \Cref{prop:vi-zero-means-equiv}).

\section{Experimental Results}
\label{sec:experimental-results}
In our experiments, we aim to minimize a reparameterized version of the CIB, given by $ \LCIBgamma[q_{T\mid X}] \vcentcolon= (1-\gamma)I(X;T) - \gamma I_{c}(Y\mid \dop(T)) $, where $\gamma \in [0, 1]$ is the trade-off parameter.
By a slight abuse of notation, we distinguish this from the original parameterization of the CIB solely by the superscript $\gamma$.
The parameter $\gamma$ has a more intuitive interpretation than $\beta$, representing the fraction of the CIB that the interventional sufficiency term accounts for.
Additionally, it simplifies hyperparameter searches when using optimization algorithms, since the magnitude of the values of $\LCIBgamma$ remains relatively stable with variations in $\gamma$, unlike what happens with $\LCIB$ and $\beta$.
It is straightforward to verify that minimizing the CIB $\LCIB$ is equivalent to minimizing its reparameterization $\LCIBgamma$, provided that $\gamma$ is selected appropriately (see \Cref{prop:wcib-equiv-to-cib}).

In this section, we will demonstrate the application of the CIB Lagrangian by finding its minimum for three problems of increasing complexity and for different values of $\gamma$, using the local search algorithms\footnotemark (Simplex) Projected Gradient Descent (pGD) and (Simplex) Projected Simulated Annealing Gradient Descent (pSAGD) described in \Cref{sec:learning_algo}, and checking whether the results are what we expect.
\footnotetext{The learning rates were chosen based on hyperparameter searches for the $\gamma=1$ cases as well as the typical values of the gradient norm.}

In each experiment, the data is generated by a different SCM, and the minimization algorithm relies solely on the \emph{observational distributions} generated by the SCM and knowledge of a \emph{backdoor set} which can be conditioned on (\Cref{assump:backdoor}).
Notably, the structural assignments are not utilized.

In each experiment \emph{we will check whether}:
\begin{enumerate}[label=(\alph*)]
  \item If $\gamma=1$, the learned v-abstraction $T$ coincides (modulo $\cong$) with the ground truth $\underline{T}$ of that experiment. That is, whether $\VI(T, \underline{T}) = 0$.
  \item If $\gamma=0$, the learned abstraction has $I(X;T) = 0$.
  \item Larger $\gamma$ values correspond to larger (or at least not smaller) values of $I_{c}(Y\mid \dop(T))$. That is, if $\gamma_1 > \gamma_2$, one has that the causal information gain for the encoder learned when $\gamma = \gamma_1$ is larger or equal to that of the encoder learned when $\gamma = \gamma_2$.
\end{enumerate}

If (a), (b), and (c) hold, this provides evidence that the CIB can be used to learn abstractions that maximize control (by setting $\gamma = 1$), maximize compression (by setting $\gamma = 0$), or strike a balance between the two (by setting $\gamma \in (0,1)$).
This is also evidence that the proposed local search algorithms succeed in optimizing the CIB objective.
Note that ground truth is only available for $\gamma = 1$ and $\gamma = 0$.
For $\gamma = 1$, the optimal solution will be clear for the proposed case studies.
For $\gamma = 0$, the solution should be maximally compressive, regardless of causal control over $Y$.
For $\gamma \notin \{0,1\}$, there is no obvious ground truth, but the reasonableness of the results can still be assessed as described in (c).
It is also noteworthy that encoders trained with different $\gamma$ values may achieve the same $I_{c}(Y\mid \dop(T))$, although increasing $\gamma$ should not decrease $I_{c}(Y\mid \dop(T))$.
This property is analogous to the Information Bottleneck (IB) framework, where distinct $\beta$ values often yield the same sufficiency value \citep{kolchinsky2018caveats}.

\paragraph{Learning Odd and Even}
Consider the scenario where the parity of $X$ determines the outcome $Y$ with some uncertainty, parameterized by $u_{Y}$ (see \Cref{fig:odd-and-even}, \Cref{sec:appendix-exp-details}).

To preserve the control that $X$ has over $Y$, an abstraction $T$ of $X$ should reflect the parity of $X$.
Consequently, when $T$ is binary and we aim to maximize the causal control of $T$ over $Y$, $T$ must be equivalent to $\underline{T} = X \mod 2$.
This will serve as the ground truth (modulo $\cong$) for the case $\gamma=1$.

\emph{Results:}
For each $\gamma$ value, we employ an ensemble of $4$ pGD optimizers with learning rates $1.0$ and $0.1$.
As shown in \Cref{subfig:odd-even-uy02}, the experiment satisfies all checks (a), (b), and (c).
For the particularly relevant $\gamma=1$ case, the full-control ground truth was found in $100\%$ of the runs.

\newcommand{\imagescale}{0.24}
\begin{figure*}
  \centering
  \begin{subfigure}[t]{0.32\textwidth}
    \centering
    \includegraphics[trim=50 0 0 0, clip, scale=\imagescale]{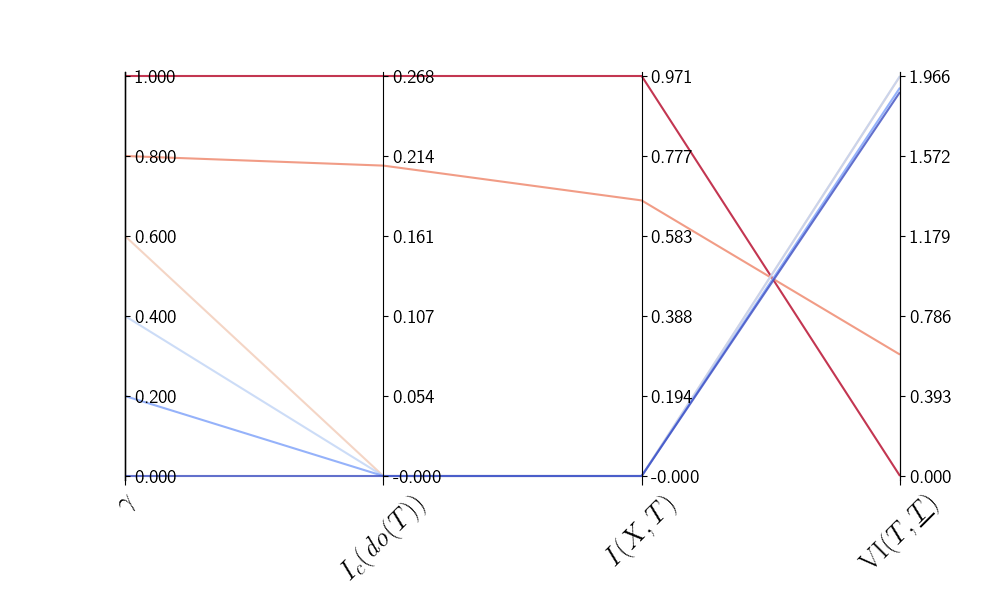}
    \caption{Odd and Even experiment with $u_{Y} = 0.2$. }
    \label{subfig:odd-even-uy02}
  \end{subfigure}
  \hfill
  \begin{subfigure}[t]{0.32\textwidth}
    \centering
    \includegraphics[trim=50 0 0 0, clip, scale=\imagescale]{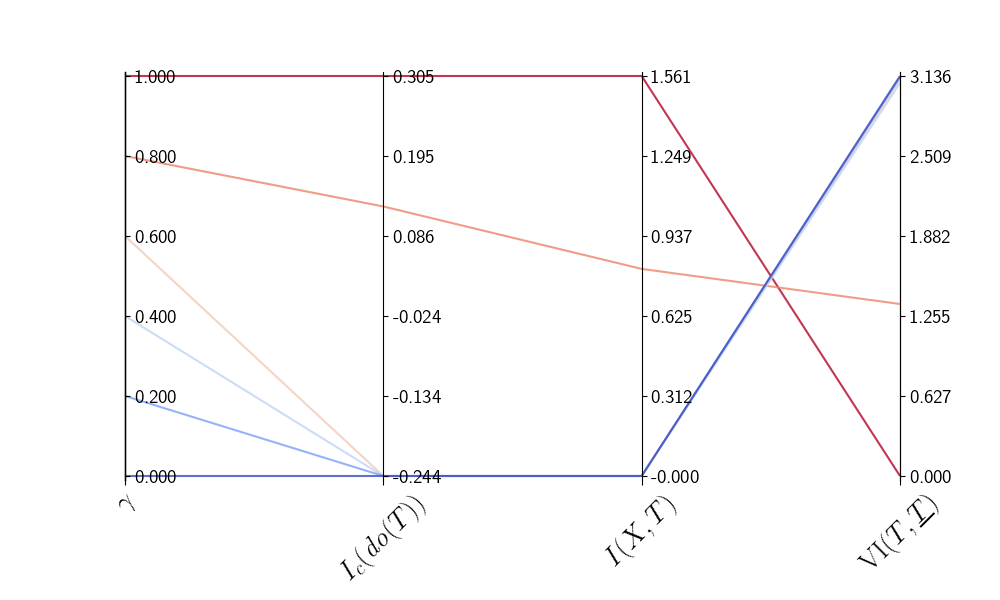}
    \caption{Confounded Addition experiment with $r_{Y} = 0.5$.}
    \label{subfig:conf-add-ry05}
  \end{subfigure}
  \hfill
  \begin{subfigure}[t]{0.32\textwidth}
    \centering
    \includegraphics[trim=50 0 0 0, clip, scale=\imagescale]{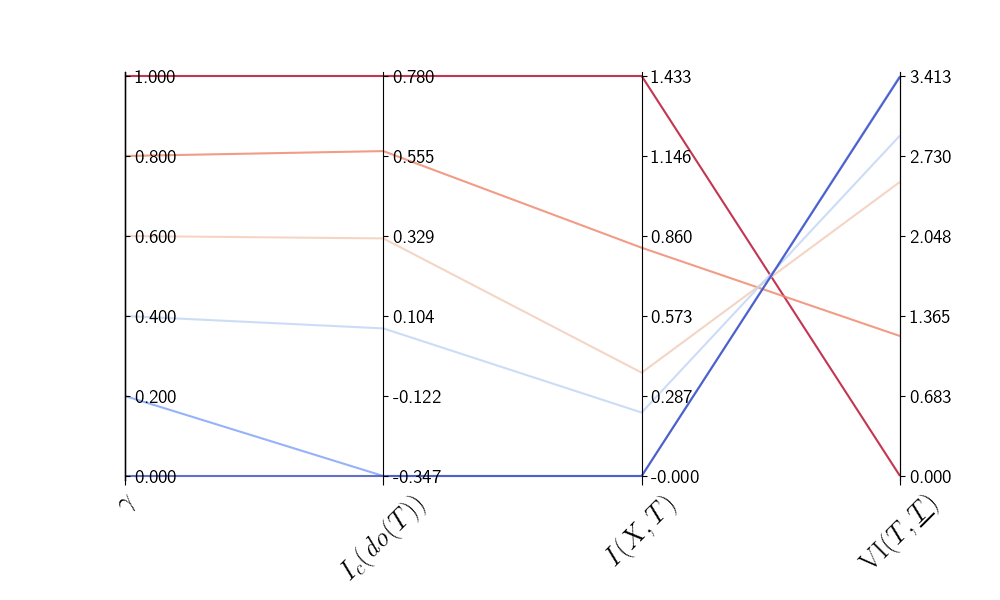}
    \caption{Genetic Mutations experiment with $b_{X_{i}} = 0.3$, $b_{Y} = 0.1$ and $b_{S} = 0.5$.}
    \label{subfig:mutations}
  \end{subfigure}
  \caption{Experimental results for specific noise distributions (other choices were also tested, with similar results --- see \Cref{sec:appendix-other-noises}). Each line corresponds to the abstraction found for the chosen $\gamma \in \{0, 0.2, \ldots, 1.0\}$.
  Maximal $\gamma$ leads to maximal causal control and learning the ground truth, while minimal $\gamma$ leads to maximal compression, as expected. Larger $\gamma$ values result in larger (or at least not smaller) causal control. Notice that some lines overlap.}
  \label{fig:exp-results-fixed-param}
\end{figure*}

This case is straightforward, with $T$ abstracting a single variable and no confounders.
The backdoor criterion is trivially satisfied by the empty set, and the Conditional Information Bottleneck (CIB) reduces to the standard Information Bottleneck (IB).
We now consider a more complex case where controlling for confounding is crucial.

\paragraph{Learning Addition in the Presence of Strong Confounding}
Consider the SCM in \Cref{subfig:conf-add}, which represents a situation where $Y$ is controlled by $X = (X_{1}, X_{2})$ through the sum $X_{1} + X_{2} \in \{0, 1, 2\}$, and $W$ confounds $X_{1}$ and $Y$.
Notice that $W$ satisfies the backdoor criterion relative to $(X,Y)$.
To preserve the control that $X$ has over $Y$, an abstraction $T$ of $X$ should keep the value of $X_{1} + X_{2} \in \{0, 1, 2\}$.
This is because, by construction of the Structural Causal Model (SCM), the sum of $X_{1}$ and $X_{2}$ is the only aspect of $X$ that can be manipulated to affect $Y$.
Therefore, if $T$ is chosen to be a 3-valued variable and one aims to maximize causal control of $T$ over $Y$, it follows that $T$ should be equivalent to $\underline{T} = X_{1} + X_{2}$.
This will be the ground truth abstraction (modulo $\cong$) for the case $\gamma=1$.

\emph{Results:}
For each value of $\gamma$, we employ an ensemble of $6$ pSAGD optimizers with a temperature of $10.0$, learning rates $1.0$ and $10.0$.
As shown in \Cref{subfig:conf-add-ry05}, the experiment satisfies all checks (a), (b), and (c).
This shows in particular that our method successfully deals with the confounding effect of $W$.
For the particularly relevant $\gamma=1$ case, the full-control ground truth was found in $98\%$ of the runs.

\begin{figure*}[h]
  \centering
    \begin{subfigure}[t]{0.49\textwidth}
      \centering
      \scalebox{\diagramscale}{
        \begin{tikzpicture}[mynode/.style={circle,draw=black,fill=white,inner sep=0pt,minimum size=0.8cm}]
          \node[mynode] (x1) at (0,0) { $X_{1}$};
          \node[mynode] (x2) at (0,-1.5) { $X_{2}$};
          \node[mynode] (y) at (2,-0.75) { $Y$};
          \node[mynode] (w) at (1.5,0.75) { $W$};
          \path [draw,->] (x1) edge[-latex] (y);
          \path [draw,->] (x2) edge[-latex] (y);
          \path [draw,->] (x2) edge[-latex] (x1);
          \path [draw,->] (w) edge[-latex] (x1);
          \path [draw,->] (w) edge[-latex] (y);
          \node (assigns) at (6,-0.5)
          {
            $
            \begin{cases}
              Y \vcentcolon= W + (X_{1} + X_{2})N_{Y} \\
              X_{1} \vcentcolon= \min(N_{X_{1}} + \nicefrac{W}{3} \cdot X_{2}, 1) \\
              X_{2} \vcentcolon= N_{X_{2}} \\
              W \vcentcolon= 3N_{W} \\
              N_{X_{i}}, N_{W} \sim \mathrm{Bern}(\nicefrac{1}{2}) \\
              N_{Y} \sim \mathrm{Bern}(r_{Y}), r_{Y} > 0
            \end{cases}
            $
          };
        \end{tikzpicture}
      } %
      \caption{SCM for the Confounded Addition experiment.
        Notice that $Y\in \{0, 1, 2, 3, 4, 5\}$ and $W\in \{0, 3\}$.}
      \label{subfig:conf-add}
    \end{subfigure}
    \hfill
    \begin{subfigure}[t]{0.49\textwidth}
        \centering
        \scalebox{\diagramscale}{
          \begin{tikzpicture}[mynode/.style={circle,draw=black,fill=white,inner sep=0pt,minimum size=0.8cm}]
              \node[mynode] (x1) at (-4.0,0) { $X_{1}$};
              \node[mynode] (x2) at (-3.0,0) { $X_{2}$};
              \node[mynode] (x3) at (-2.0,0) { $X_{3}$};
              \node[mynode] (x4) at (-1,0) { $X_{4}$};
              \node[mynode] (y) at (-3.5,-1.5) { $Y$};
              \node[mynode] (s) at (-1.5,-1.5) { $S$};
              \path [draw,->] (x1) edge[-latex] (y);
              \path [draw,->] (x2) edge[-latex] (y);
              \path [draw,->] (x3) edge[-latex] (y);
              \path [draw,->] (s) edge[-latex] (x3);
              \path [draw,->] (s) edge[-latex] (x4);
              \path [draw,->] (s) edge[-latex] (y);
              \node (assigns) at (4,-0.75)
                {
                  $
                  \begin{cases}
                    S \vcentcolon= N_{S} \\
                    X_{1} \vcentcolon= N_{X_{1}} \\
                    X_{2} \vcentcolon= N_{X_{2}} \\
                    X_{3} \vcentcolon= S \cdot N_{X_{3}} + (1-S)\cdot(1-N_{X_{3}}) \\
                    X_{4} \vcentcolon= S \cdot N_{X_{4}} + (1-S)\cdot(1-N_{X_{4}}) \\
                    Y \vcentcolon= S + X_{1} + X_{2} + (2-X_{3}) \cdot X_{1}\cdot X_{2} + N_{Y} \\
                  \end{cases}
                  $
                };
          \end{tikzpicture}
        } %
        \caption{SCM for the Genetic Mutations experiment.
          All noise variables $N_{V}$ follow a Bernoulli distribution $\mathrm{Bern}(b_{V})$.
        }
        \label{fig:mutations-example}
    \end{subfigure}
    \label{subfig:mutations-example}
    \caption{SCMs for two of the experiments. The structural equations are \emph{not} known to the algorithm. Our experiments show that the OCVAs provide us with information about the causal relations.}
  \label{fig:scms}
\end{figure*}

\paragraph{Interaction between Genetic Mutations}
Consider the SCM depicted in \Cref{fig:mutations-example}, which represents a scenario of genetic mutations in mice like the one described in \Cref{sec:introd}.
Notice that $S$ satisfies the backdoor criterion relative to $(X,Y)$.
By inspecting the structural assignments in \Cref{fig:mutations-example}, we can see that the mutations interact in a complex, non-additive way with respect to the body mass $Y$.
Specifically, individual mutations at $s_{1}$ or $s_{2}$ have an equivalent, relatively small effect on $Y$, while having both mutations simultaneously would have a profound impact, larger than the sum of the impacts of the individual mutations.
Furthermore, having a mutation at $s_{3}$ has no effect on its own, but partially protects against the effect of simultaneous mutations at $s_{1}$ and $s_{2}$.
Finally, $X_{4}$ has no effect on $Y$ whatsoever.
This means in particular that $T$ should be able to distinguish cases where both $s_{1}$ and $s_{2}$ are mutated but $s_{3}$ is not, and where all three are mutated.

This setting is characterized by an increased number of variables, confounding, and complex epistatic interactions between variables, which the learned encoder must capture.
Furthermore, the encoder should recognize that the maximal control solution does not need to account for one of the variables in $X$ (specifically, $X_{4}$).
Denote by $\theta(X)$ the sum $X_{1} + X_{2} + (2-X_{3}) \cdot X_{1}\cdot X_{2}$ of all the terms in the structural assignment of $Y$ which involve $X$.
This expression represents the aspect of $Y$ that $X$ can influence, and thus it is the information that $T$ should capture to achieve maximal control.
It follows that a 4-valued v-abstraction learned with $\gamma=1$ case should be equivalent to the v-abstraction $\underline{T}$ with encoder $q_{\underline{T}\mid X}$ given by $q_{\underline{T}\mid X}(t\mid x_{1}, x_{2}, x_{3}, x_{4}) = \delta_{t, \theta(x_{1}, x_{2},x_{3},x_{4})}$.
This will be the ground truth abstraction (modulo $\cong$) for the case $\gamma=1$.

\emph{Results:}
For each value of $\gamma$, we employ an ensemble of $12$ pSAGD optimizers with a temperature of $10.0$, learning rates $10^{0}$, $10^{1}$, $10^{2}$ and $10^{4}$.
As shown in \Cref{subfig:mutations}, the experiment satisfies all checks (a), (b), and (c).
This shows that our method can capture complex interaction effects between the variables, while simultaneously dealing with the confounding effect of $S$.
For the particularly relevant $\gamma=1$ case, the full-control ground truth was found in $95\%$ of the runs.

We emphasize that the full-control OCVAs provide information about the causal relationships which is not clear from knowledge of the joint distribution and backdoor set (which are the only inputs of the method).
For instance, in the Genetic Mutations experiment, looking at the OCVA learned for $\gamma=1$ (the ground truth $\underline{T}$) reveals, for example, that a mutation on $s_{3}$ has a protective effect which is activated (only) when both $s_{1}$ and $s_{2}$ are mutated.
This is a complex causal interaction between the variables which would not be clear from the input data.
Knowledge of the OCVA can also help with intervention selection tasks: in this example, one would not intervene on $X_{3}$ in a case where one of the $s_{1}, s_{2}$ is not mutated.
Even in cases where the structural equations are known, they are often complex and hard to interpret directly, in which case learning OCVAs can again help with interpretability and intervention selection.

We also ran the IB method on the same datasets to compare with the CIB method.
As expected, the IB method failed to learn the (full-control) ground-truth when confounding was present.
A detailed analysis can be found in \Cref{sec:appendix-IB}.

\section{Related Work}
\label{sec:related-work}
There has been one other line of work investigating a problem related to our search for an OCVA.
Namely, in \citet{chalupka2017overview,chalupka2014visual,chalupka2016multi,chalupka2016unsupervised}), the authors search for ``causal macrovariables'' of higher-dimensional ``microvariables'', described as coarse representations of the microvariable $X$ which preserve the causal relation between $X$ and $Y$\footnotemark.
\footnotetext{In general, they search for abstractions not only of $X$ but also of $Y$. We are only interested in how they construct abstractions of $X$.}
Their definition of macrovariable is distinct from our definition of OCVA.
Namely, it is not based on information-theoretical concepts such as compression and sufficiency, but on clustering together the values of $X$ \emph{resulting in the same post-intervention distributions} over $Y$.
Each such cluster is called a ``causal class''.
Our method \emph{subsumes} theirs, in the sense that causal macrovariables coincide with the OCVAs in the full-control cases.
To be precise, for the full-control ($\gamma=1$) case and $\abs{R_T}$ equal to the number $N_{C}$ of causal classes, the CIB method learns a deterministic encoder mapping all $x$ with the same post-intervention distribution $p_Y^{\dop(X=x)}$ to the same $t$ (\Cref{prop:cib_and_macrovariables}, \Cref{sec:appendix-cib-and-macrovariables}).
Hence, the full-control OCVAs with $\abs{R_{T}}=N_{C}$ are exactly the causal macrovariables from \citet{chalupka2016unsupervised}.
\Citet{holtgen2021encoding,aruna2023semi}
tried to address the problem of finding causal macrovariables using the standard information bottleneck framework.
However, their method does not account for causality, being limited to scenarios without confounding due to its reliance on the standard IB, which, as we checked experimentally, can fail in the presence of confounding.

The study of abstractions in causality has been conducted from yet another point of view.
Namely, there have been efforts to formalize when exactly a given causal model can be seen as a causally consistent abstraction of another one \citep{rubenstein2017causal,beckers2019abstracting,beckers2020approximate}.
The core idea is that intervening on a variable in the low-level model and then transitioning to the high-level model should yield the same distribution over the model variables as first transitioning to the high-level model and then intervening.
Although these works do not provide methods for constructing causal model abstractions, recent efforts have focused on learning such high-level models \citet{xia2024neural,zennaro2023jointly,felekis2024causal}.
In contrast, our method does not learn an abstraction of an SCM, but a an abstraction of a (set of) variable(s) $X$ which is causally relevant.

Finally, we note that the learning task in the causal representation learning literature (see \emph{e.g.} \citet{scholkopf2021towards,brehmer2022weakly,bonifati2022time2feat,ahuja2023interventional,lippe2023biscuit,bing2023invariance,von2024nonparametric,bing2024identifying,buchholz2024learning}) may appear to be related to to our task of learning OCVAs.
However, this is not the case.
In the causal representation learning task, the goal is to infer latent causal variables (and their causal relations) which may have generated some given low-level data.
In contrast, our method aims to construct a an abstraction $T$ of a variable $X$ which keeps causal control over a chosen target $Y$ while compressing $X$.
In particular, we do not make any assumptions of existence of latent variables from which data is generated.

\section{Conclusion and Future Work}
\label{sec:conclusion}

We extended the notion of optimal variable abstraction to the causal setting, resulting in an axiomatic characterization of optimal causal variable abstractions.
Just as the information bottleneck (IB) method can be used to learn optimal variable abstractions, so can the causal information bottleneck (CIB) method introduced in this paper be used to learn optimal causal variable abstractions.

The CIB, which depends on the interventional distributions $p(y\mid \dop(t))$, needs to be computed during the learning procedure, which consists of solving a constrained minimization problem.
The exact expression for the CIB depends on the causal structure of the system under study.
We focused on cases where there is a set $\mathbf{Z}$ satisfying the backdoor criterion relative to $(X,Y)$ among the observed variables.
This allowed us to derive a backdoor adjustment formula for $p(y\mid \dop(t))$, and thus successfully apply a minimization algorithm to minimize the CIB.
Specifically, we introduced a local search algorithm, referred to as projected simulated annealing gradient descent (pSAGD), which integrates simulated annealing and gradient descent techniques with a projection operator to maintain constraint satisfaction throughout the minimization process.
In order to compare different variable abstractions (v-abstractions) learned by our algorithm, we introduced a notion of equivalence of v-abstractions, which partitions v-abstractions into abstraction classes, and showed that the variation of information can be used to assess whether two v-abstractions are equivalent.
We experimentally validated that the learned v-abstractions in three toy models of increasing complexity align with our expectations, and that the standard information bottleneck method fails to the the same.

Future research directions include exploring alternative methods for incorporating causality into the information bottleneck framework, such as focusing on causal properties other than causal control, like proportionality \citep{pocheville2015comparing}.
Another area worth investigating relates to fine-tuning the trade-off between compression and interventional sufficiency.
\Citet{kolchinsky2018caveats} highlight that, in the context of the Information Bottleneck (IB), different values of $\beta$ can often result in the same sufficiency.
Similarly, in our experiments we observed that different values of $\gamma$ (and thus $\beta$) often produced the same interventional sufficiency values.
Future work could explore strategies similar to those used by \citet{kolchinsky2018caveats} to address this.
Additionally, another natural next step would be to adapt the causal information bottleneck (CIB) method to continuous variables, for example by using variational autoencoders \citep{kingma2013variational} to minimize the CIB Lagrangian, as was previously done for the standard IB \citep{alemi2016deep}.
One can also explore the relationship between our v-abstraction learning method and the framework of causal abstractions, similarly to how \citet{beckers2020approximate} connect the latter with the approach from \citet{chalupka2017overview}.
Finally, while the local search algorithms used in our experiments (pGD and pSAGD) sufficed to demonstrate the applicability of the CIB method for smaller datasets, scaling to larger or more complex cases may require modifications to deal with local minima. Developing such algorithms for the CIB method is left to future work.

\section*{Acknowledgments}
\label{sec:acknowledgments}
This publication is part of the CAUSES project (KIVI.2019.004) of the research programme Responsible Use of Artificial Intelligence which is financed by the Dutch Research Council (NWO) and ProRail.
We thank Yorgos Felekis for the useful discussions about related work.

\bibliography{cib_paper.bib}

\newpage

\onecolumn
\appendix

\title{Appendix}
\maketitle

\section{More Preliminaries}
\label{sec:appendix-more-prelim}

\subsection{Entropy and Mutual Information}
\label{sec:appendix-entr-mutu-inform}
In this subsection we will state the definitions of entropy, conditional
entropy and mutual information. In the interest of space, we will not try to
motivate these definitions.
For more, see \citet{thomas2006information}.
\begin{definition}[Entropy and Cond. Entropy
  \citep{thomas2006information}]
  \label[definition]{def:entr}
  Let $X$ be a discrete random variable with range $R_{X}$ and $p$ be a
  probability distribution for $X$. The \emph{entropy of $X$ w.r.t. the
    distribution $p$} is\footnotemark \footnotetext{In this article, $\log$
    denotes the logarithm to the base $2$.}
  \begin{equation}
    \label{eq:entr}
    H_{X \sim p}(X) \vcentcolon= -\sum_{x\in R_{X}} p(x) \log p(x).
  \end{equation}
  Entropy is measured in $\mathrm{bit}$.
  If the context suggests a canonical probability distribution for $X$, one can write $H(X)$ and refers to it simply as the \emph{entropy of $X$}. \\
  The \emph{conditional entropy} $H(Y\mid X)$ of $Y$ conditioned on $X$ is the
  expected value w.r.t. $p_{X}$ of the entropy
  $H(Y \mid X=x)\vcentcolon=H_{Y\sim p_{Y\mid
      X=x}}(Y)$:%
  \begin{equation}
    \label{eq:cond-entr}
    H(Y\mid X) \vcentcolon= \E_{x\sim p_{X}} \left[ H(Y \mid X=x) \right].
  \end{equation}
\end{definition}
This means that the conditional entropy $H(Y \mid X)$ is the entropy of $H(Y)$
that remains on average if one conditions on $X$.

There are two common equivalent ways to define mutual information (often called
information gain).

\begin{definition}[Mutual Information and Cond. Mutual Information
  \citep{thomas2006information}]
  \label[definition]{def:mutual-information}
  Let $X$ and $Y$ be discrete random variables with ranges $R_{X}$ and $R_{Y}$
  and distributions $p_{X}$ and $p_{Y}$, respectively. The \emph{mutual
    information} between $X$ and $Y$ is
  \begin{equation}
    \label{eq:mi-independence-form}
    I(X; Y) := \!\!\!\! \sum_{x, y \in R_{X} \times R_{Y}} \!\!\!\! p_{X, Y}(x, y) \log \frac{p_{X, Y}(x, y)}{p_{X}(x) p_{Y}(y)}.
  \end{equation}
  Or equivalently:
  \begin{align}
    \begin{split}
      \label{eq:mi-entr-form}
      I(X; Y) &:= H(Y) - H(Y \mid X) \\
      &= H(X) - H(X \mid Y).
    \end{split}
  \end{align}
  Let $Z$ be another discrete random variable. The \emph{conditional mutual
    information} between $X$ and $Y$ conditioned on $Z$ is:
  \begin{align}
    \begin{split}
      \label{eq:cond-mut-info}
      I(X; Y \mid Z) &:= H(Y \mid Z) - H(Y \mid X, Z) \\
                    &= H(X \mid Z) - H(X \mid Y, Z).
    \end{split}
  \end{align}
\end{definition}

The view of mutual information as entropy reduction from \Cref{eq:mi-entr-form}
is the starting point for the definition of causal information gain.

\subsection{More on Causal Entropy and Causal Information Gain}
\label{sec:appendix-caus-entr-caus}

In this section, we will define causal entropy and causal information gain.
The latter will be an essential component of our method.
See \citet{simoes2023causal} for a thorough discussion about these concepts.
Let $X$ and $Y$ be endogenous variables of an SCM $\C$.
The causal entropy of $Y$ for $X$ is the entropy of $Y$ that remains, on
average, after one atomically intervenes on $X$.
Its definition is analogous to that of conditional entropy (see \Cref{def:entr}).
Concretely, causal entropy is the average uncertainty one has about $Y$ if one sets $X$ to $x$ with probability $p_{X^{\star}}(x)$, where $p_{X^{\star}}$ specifies the distribution over interventions.
\begin{definition}[Causal Entropy, $H_{c}$ \citep{simoes2023causal}]
  Let $Y$, $X$ and $X^{\star}$ be random variables such that $X$ and $X^{\star}$ have the same
  range and $X^{\star}$ is independent of all variables in $\cC$. We say that $X^{\star}$ is
  an \emph{intervention protocol} for $X$.
  The \emph{causal entropy} $H_{c}(Y\mid \dop(X \sim X^{\star}))$ of $Y$ for $X$ given the
  intervention protocol $X^{\star}$ is the expected value w.r.t. $p_{X^{\star}}$ of
  the entropy
  $H(Y \mid \dop(X = x)) \vcentcolon= H_{Y \sim p_{Y}^{\dop(X=x)}}(Y)$ of the
  interventional distribution $p_{Y}^{\dop(X=x)}$. That is:
  \begin{equation}
    \label{eq:caus-cond-entr}
    H_{c}(Y\mid \dop(X \sim  X^{\star})) \vcentcolon= \E_{x\sim p_{X^{\star}}} \left[ H(Y \mid \dop(X=x)) \right].
  \end{equation}
\end{definition}

Causal information gain extends mutual information/information gain to the
causal context. While mutual information between two variables $X$ and $Y$ is
the average reduction in uncertainty about $Y$ if one observes the value of $X$
(see \Cref{eq:mi-entr-form}), the causal information gain of $Y$ for $X$ is
the average decrease in the entropy of $Y$ after one atomically intervenes on
$X$ (following an intervention protocol $X^{\star}$).
\begin{definition}[Causal Information Gain, $I_{c}$ \citep{simoes2023causal}]
  \label[definition]{def:Ic}
  Let $Y$, $X$ and $X^{\star}$ be random variables such that $X^{\star}$ is an intervention
  protocol for $X$. The \emph{causal information gain}
  $I_{c}(Y\mid \dop(X \sim X^{\star}))$ of $Y$ for $X$ given the intervention protocol
  $X^{\star}$ is the difference between the entropy of $Y$ w.r.t. its prior and the
  causal entropy of $Y$ for $X$ given the intervention protocol $X^{\star}$. That is:
  \begin{equation}
    \label{eq:caus-cond-entr}
    I_{c}(Y \mid \dop(X \sim X^{\star})) \vcentcolon= H(Y) - H_{c}(Y\mid \dop(X \sim  X^{\star})).
  \end{equation}
\end{definition}
The causal information gain of $Y$ for $X$ was proposed in \citet{simoes2023causal} as a measure of the ``(causal) control that variable $X$ has over the variable $Y$'', that is, the reduction of uncertainty about $Y$ that results from intervening on $X$.
It is noteworthy that $I_{c}(Y \mid \dop(X))$ can be negative, in contrast with mutual information.

As described in \Cref{sec:interv-abst}, we choose a uniform prior over the interventions on $X$; that is, a uniform protocol $p_{X^{\star}} = p^{\star}$ over the low-level variables $X$.
This induces an intervention protocol $p^{\star}(t) = \sum_{x} q(t\mid x) p^{\star}(x)$ for $T$.
To simplify notation, we omit the protocol and write simply $I_{c}(Y\mid \dop (X))$ and $I_{c}(Y\mid \dop (T))$.

\subsection{More on Structural Causal Models}
\label{sec:struct-caus-models}

One can model the causal structure of a system by means of a ``structural causal model'', which can be seen as a Bayesian network \citep{koller2009probabilistic} whose graph $G$ has a causal interpretation and each conditional probability distribution (CPD) $P(X_{i} \mid \Pa_{X_{i}})$ of the Bayesian network stems from a deterministic function $f_{X_{i}}$ (called ``structural assignment'') of the parents of $X_{i}$.
In this context, it is common to separate the parent-less random variables (which are called ``exogenous'' or ``noise'' variables) from the rest (called ``endogenous'' variables).
Only the endogenous variables are represented in the structural causal model graph.
As is commonly done \citep{peters2017elements}, we assume that the noise variables are jointly independent and that exactly one noise variable $N_{X_{i}}$ appears as an argument in the structural assignment $f_{X_{i}}$ of $X_{i}$.
In full rigor\footnotemark \citep{peters2017elements}:

\footnotetext{We slightly rephrase the definition provided in \citet{peters2017elements} for our purposes. \label{fn:def}}

\begin{definition}[Structural Causal Model]
  \label[definition]{def:scm}
  Let $X$ be a random variable with range $R_{X}$ and $\mathbf{W}$ a random vector with range $R_{\mathbf{W}}$.
  A \emph{structural assignment for $X$ from $\mathbf{W}$} is a function $f_{X}\colon R_{\mathbf{W}} \to R_{X}$.
  A \emph{structural causal model (SCM)} $\mathfrak{C} = (\mathbf{X}, \mathbf{N}, S, p_{\mathbf{N}})$ consists of:
  \begin{enumerate}
    \item A random vector $\mathbf{X} = (X_{1}, \ldots, X_{n})$ whose variables we call \emph{endogenous}.
    \item A random vector $\mathbf{N} = (N_{X_{1}}, \ldots, N_{X_{n}})$ whose variables we call \emph{exogenous} or \emph{noise}.
    \item A set $S$ of $n$ structural assignments $f_{X_{i}}$ for $X_{i}$ from ($\Pa_{X_{i}}, N_{X_{i}}$), where $\Pa_{X_{i}} \subseteq \mathbf{X}$ are called \emph{parents} of $X_{i}$.
      The \emph{causal graph} $G^{\mathfrak{C}}\vcentcolon=(\mathbf{X}, E)$ of $\mathfrak{C}$ has as its edge set $E = \{(P, X_{i}) : X_{i} \in \mathbf{X},\  P\in \Pa_{X_{i}}\}$.
      The $\Pa_{X_{i}}$ must be such that the $G^{\mathfrak{C}}$ is a directed acyclic graph (DAG).
    \item A jointly independent probability distribution $p_{\mathbf{N}}$ over the noise variables. We call it simply the \emph{noise distribution}.
  \end{enumerate}
\end{definition}

We denote by $\C(\mathbf{X})$ the set of SCMs with vector of endogenous variables $\mathbf{X}$.
In general, we allow $\mathbf{X}$ to consist of both observed variables $\mathbf{O}$ and hidden variables $\mathbf{H}$, so that $\mathbf{X} = \mathbf{O} \cup \mathbf{H}$.
Notice that for a given SCM the noise variables have a known distribution $p_{\mathbf{N}}$ and the endogenous variables can be written as functions of the noise variables.
Therefore the distributions of the endogenous variables are themselves determined if one fixes the SCM.
This brings us to the notion of the entailed distribution\footref{fn:def} \citep{peters2017elements}:

\begin{definition}[Entailed distribution]
  Let $\mathfrak{C} = (\mathbf{X}, \mathbf{N}, S, p_{\mathbf{N}})$ be an SCM. Its \emph{entailed distribution} $p^{\mathfrak{C}}_{\mathbf{X}}$  is the unique joint distribution over $\mathbf{X}$ such that $\forall X_{i} \in \mathbf{X},\ X_{i} = f_{X_{i}}(\Pa_{X_{i}}, N_{X_{i}})$.
  It is often simply denoted by $p^{\cC}$.
  Let $\mathbf{x}_{-i}\vcentcolon= (x_{1}, \ldots, x_{i-1}, x_{i+1}, \ldots, x_{n})$.
  For a given $X_{i} \in \mathbf{X}$, the marginalized distribution $p^{\cC}_{X_{i}}$ given by $p^{\cC}_{X_{i}}(x_{i}) = \sum_{\mathbf{x}_{-i}} p^{\cC}_{\mathbf{X}}(\mathbf{x})$ is also referred to as \emph{entailed distribution (of $X_{i}$)}.
\end{definition}

An SCM allows us to model interventions on the system.
The idea is that an SCM represents how the values of the random variables are generated, and by intervening on a variable we are effectively changing its generating process.
Thus intervening on a variable can be modeled by modifying the structural assignment of said variable, resulting in a new SCM differing from the original only in the structural assignment of the intervened variable, and possibly introducing a new noise variable for it, in place of the old one.
Naturally, the new SCM will have an entailed distribution which is in general different from the distribution entailed by the original SCM.

The most common type of interventions are the so-called ``atomic interventions'', where one sets a variable to a chosen value, effectively replacing the distribution of the intervened variable with a point mass distribution.
In particular, this means that the intervened variable has no parents after the intervention.
This is the only type of intervention that we will need to consider in this work.
Formally\footref{fn:def} \cite{peters2017elements}:
\begin{definition}[Atomic intervention]
  Let $\cC = (\mathbf{X}, \mathbf{N}, S, p_{\mathbf{N}})$ be an SCM, $X_{i} \in \mathbf{X}$ and $x\in R_{X_{i}}$.
  The \emph{atomic intervention} $\dop(X_{i}=x)$ is the function $\cC(\mathbf{X}) \to \cC(\mathbf{X})$ given by $\cC \mapsto \cC^{\dop(X_{i} = x)}$, where $\cC^{\dop(X_{i} = x)}$ is the SCM that differs from $\cC$ only in that the structural assignment $f_{X_{i}}(\Pa_{X_{i}}, N_{X_{i}})$ is replaced by the structural assignment $\tilde{f}_{X_{i}}(\tilde{N}_{X_{i}}) = \tilde{N}_{X_{i}}$, where $\tilde{N}_{X_{i}}$ is a random variable with range $R_{X_{i}}$ and\footnotemark $p_{\tilde{N}_{X_{i}}}(x_{i}) = \mathbf{1}_{x}(x_{i})$ for all $x_{i} \in R_{X_{i}}$.
\footnotetext{We denote by $\mathbf{1}_{x}$ the indicator function of $x$, so that $\mathbf{1}_{x}(x_{i}) =
  \begin{cases}
    1,&  x_{i} = x \\
    0,& \mathrm{otherwise}
  \end{cases}
  $. }
  Such SCM is called the \emph{post-atomic-intervention SCM}.
  One says that the variable $X_{i}$ was \emph{(atomically) intervened on}.
  The distribution $p^{\dop(X_{i} = x)} \vcentcolon= p^{\cC^{\dop(X_{i}=x)}}$ entailed by $\cC^{\dop(X_{i}= x)}$ is called the \emph{post-intervention distribution (w.r.t. the atomic intervention $\dop(X_{i} = x)$ on $\cC$)}.
\end{definition}

\section{Supplementary Material on the CIB}

\subsection{Lagrange Multipliers and the CIB}
\begin{remark}[Distinctions from classical Lagrange multipliers]
  The minimization problem in \Cref{eq:our_min_problem} involves both equality and inequality constraints.
  To tackle this problem using the method of Lagrange multipliers \citep{nocedal2006numerical} directly, we would need to construct a Lagrangian of the form
  $\mathcal{L}(q_{T\mid X}, \beta, (\lambda_{x})_{x}, (\mu_{x,t})_{x,t}) = I(X;T) - \beta (I_{c}(Y\mid\dop(T)) - D) - \sum_{x}\lambda_xg_{x}((q_{T \mid X=x})_{x}) - \sum_{x, t}\mu_{x,t}h_{x,t}(q_{T=t \mid X=x})$,
  where the $g_{x}$ are the restriction functions ensuring that all the conditional distributions $q_{T\mid X=x}$ are normalized and the $h_{x,t}$ are the inequality restriction functions ensuring non-negativity of the conditional distributions.
  Hence the last two terms of $\mathcal{L}$, along with the appropriate Karush-Kuhn-Tucker (KKT) conditions, would guarantee that the conditional distributions $q_{T\mid X=x}$ belong to the simplex $\Delta^{\abs{R_{T}} - 1}$.
  However, finding the stationary points of the Lagrangian with respect to all its arguments, \emph{i.e.}, those where $\nabla_{q_{T\mid X}, (\lambda_{x})_{x}, (\mu_{x,t})}\mathcal{L} = 0$, would be a formidable task.
  Instead, we follow the approach of \citet{tishby2000information,strouse2017deterministic} and impose the simplex constraint separately, outside of the Lagrangian multipliers method, leaving us only with the sufficiency constraint.
  Furthermore, in contrast with the classic method of Lagrangian multipliers, the multiplier $\beta$ is fixed, so that $D$ is not chosen directly, but only indirectly through the choice of $\beta$.
\end{remark}

\subsection{Weighted Causal Information Bottleneck Lagrangian}

In this section, it will be useful to distinguish $\LCIBgamma$ from the original CIB $\LCIB$.
We call the former the \emph{weighted causal information bottleneck (wCIB) Lagrangian} and denote it $\wLCIB$.

\begin{proposition}
  \label{prop:wcib-equiv-to-cib}
  Let $\beta \in \R^{+}$ and $\gamma = \frac{\beta}{1+\beta}$.
  Then the minimizers of $\LCIB[q_{T\mid X}]$ are the same as those of $\wLCIB[q_{T\mid X}]$.
\end{proposition}
\begin{proof} %
  \begin{equation}
    \begin{split}
      \wLCIB &= (1-\gamma)I(Y;X) - \gamma I_{c}(Y\mid \dop(T)) \\
             &= \frac{1}{1+\beta}I(Y;X) - \frac{\beta}{1+\beta} I_{c}(Y\mid \dop(T)) \\
             &= \frac{1}{1+\beta} \LCIB.
    \end{split}
  \end{equation}
  Therefore, the wCIB with the chosen $\gamma$ is simply a rescaling of the CIB with scaling factor $\frac{1}{1+\beta} \in \left(0, 1\right]$.
  Since this factor is always positive, it follows that $\wLCIB$ and $\LCIB$ attain their minima at the same points.
\end{proof}

Notice that $\beta = \frac{\gamma}{1 - \gamma}$ for $\gamma < 1$, and that $\beta \rightarrow +\infty$ as $\gamma \rightarrow 1$.
In case maximal causal control of $T$ is desired without consideration of compression, we can use the wCIB with $\gamma = 1$, in which case we formally set $\beta = +\infty$.

\subsection{CIB and Causal Macrovariables}
\label{sec:appendix-cib-and-macrovariables}
In this section, we show that the CIB method learns, in the full-control ($\gamma=1$) case, the causal macrovariables from \citet{chalupka2016unsupervised,chalupka2017overview}.

\begin{definition}[Causal Classes, Partitions and Macrovariables \citep{chalupka2016unsupervised}]
  We define the following additional notation and terminology:
  \begin{enumerate}
    \item A \emph{causal class} of $X$ is an equivalence class of the equivalence relation $\sim$ on $R_{X}$ defined by $x_{1} \sim x_{2} \iff P_{Y}^{\dop(X=x_{1})} = P_{Y}^{\dop(X=x_{2})}$. A causal class is also called a \emph{(causal) macrovariable}.
    \item The \emph{causal partition} of $R_{X}$ is the partition induced by $\sim$.
  \end{enumerate}
\end{definition}

\begin{proposition}[CIB and Causal Macrovariables]
  \label{prop:cib_and_macrovariables}
  If the causal partition has $n$ causal classes, then the CIB Lagrangian with $\gamma=1$ and at least $n$ states for $T$ will be optimized by an encoder that deterministically maps each causal class to a unique value of $T$.
  If $T$ has exactly $n$ states, then the converse also holds, \emph{i.e.}, all optima correspond to the causal partition.
\end{proposition}

\begin{proof}
  For $\gamma=1$, optimizing \Cref{eq:cib_lag} is equivalent to minimizing $f_{obj}[q] = \E_{t \sim p_{T^*}}\left[H[p_{Y}^{\dop(T=t)}]\right]$. Here $p_{Y}^{\dop(T=t)}$ is defined by \Cref{eq:abst-intervention} as $\sum_x p^*(x \mid t) p_{Y}^{\dop(X=x)}$.

  We will use that $H$ is (strictly) concave, that is, for non-negative weights $w_x$ that sum to $1$, we have
  $H [ \sum_x w_x p_Y^{\dop(x)} ] \geq \sum_x w_x H [ p_Y^{\dop(x)} ]$,
  with strict inequality if there exist $x, x’$ with $w_x>0, w_{x’}>0$, and $p_Y^{\dop(x)} \neq p_Y^{\dop(x’)}$.

  Let $q(C\mid x)$ be the encoder corresponding to the causal partition.
  Write $C$ for a causal class in the causal partition.
  Then, $q(C\mid x) = \mathbf{1}_{x\in C}$.
  Note that $p^*(C) = \sum_{x} q(C\mid x) p^*(x) = \sum_{x \in C} p^*(x)$.
  This encoder achieves

  \begin{equation}
    \begin{split}
      f_{obj}[q(C\mid x)] &= \sum_C p^*(C) H [ \sum_x p^*(x \mid C) p_Y^{\dop(x)} ] \\
      &= \sum_C p^*(C) H [ \sum_x \frac{ q(C\mid x) p^*(x) }{ \sum_{x’} q(C\mid x’) p^*(x’) } p_Y^{\dop(x)} ] \\
      &= \sum_C p^*(C) H [ \sum_{x \in C} \frac{ p^*(x) }{ \sum_{x’ \in C} p^*(x’) } p_Y^{\dop(x)} ] \\
      &= \sum_C p^*(C) H [ p_Y^{\dop(x)} ] \text{(where } x \text{ is any }x \text{ in } C\text{)} \\
      &= \sum_x p^*(x) H [ p_Y^{\dop(x)} ],
  \end{split}
\label{eq:fobj_partition}
\end{equation}
where on the second step we used \Cref{eq:int-decoder}, on the third step we used that $p_{Y}^{\dop(x)}$ is the same for every $x$, and on the last step we used that the causal classes form a partition of $R_{X}$, so that $\sum_{C}\sum_{x\in C} = \sum_{x}$.
Now, consider an arbitrary v-abstraction $T$.
By using $T$ instead of the causal partition, the value of the objective function that we get is
$f_{obj}[q(t\mid x)] = \sum_t p^*(t) H [ \sum_x p^*(x \mid t) p_Y^{\dop(x)} ]$.
On the other hand, \Cref{eq:fobj_partition} can be written $f_{obj}[q(C\mid x)] = \sum_t p^*(t) \sum_x p^*(x \mid t) H [ p_Y^{\dop(x)} ]$.
By concavity of $H$ we see that every term of the sum in the expression for $f_{obj}[q(t\mid x)]$ is $\leq$ than the corresponding term (\emph{i.e.} for the same $t$) of the sum in the expression for $f_{obj}[q(C\mid x)]$.
Hence the causal partition’s objective is $\leq$ than $T$’s objective.
If $T$ has exactly $n$ states but $T$ does not correspond to the causal partition, then there must be $t$, $x$, $x’$ with $p^*(x\mid t)>0$, $p^*(x’\mid t)>0$, and $p_Y^{\dop(x)} \neq p_Y^{\dop(x’)}$, so that this inequality is strict.
\end{proof}

\section{Supplementary Material on Interventions on Variable Abstractions}
\label{sec:appendix-int-on-abst}

\begin{remark}[About the definition of Intervention Decoder]
  \label{rem:int-decoder-def}
  Notice that the intervention decoder is fixed if one has the encoder (and a choice of prior).
  This may seem counter-intuitive, as one might expect that additional information about the system is necessary to understand how interventions on a variable abstraction map $T$ to interventions on the abstracted variables $X$.
  For example, the effect of changing temperature on particle velocities may depend on the mechanism of temperature change (e.g., the type of radiator used).
  However, in our case, this concern does not arise because we are not trying to model a pre-existing variable $T$ (such as a latent variable generating the data) that has a predetermined way of influencing $X$.
  Instead, our $T$ is an artificial, conceptual abstraction of $X$.
  Thus, we can define the mechanism that propagates interventions on $T$ to $X$.
  The key requirement is that it is applied consistently.
  Our choice to ensure consistency of the intervention decoder with the encoder leads to a natural and straightforward interpretation.
  This approach allows us, for example, to determine how to intervene on genes in a genetic mutation experiment, with the decoder providing a way to map our chosen $T = t$ to the corresponding values of $X$.
\end{remark}

\subsection{Effect of a Variable Abstraction Intervention on the Target Variable}
For $Y\in \mathbf{V}$, one has that
\begin{equation}
  \label{eq:pydot}
  \begin{split}
    p^{\dop(T=t)}_{Y}(y) = p(&y\mid\dop(T=t)) =  \sum_{v_{1},\ldots,v_{m}} p(v_{1},\ldots,v_{m}, y \mid \dop(T=t)) \\
      &=  \sum_{x}p^{\star}(x\mid t) \sum_{v_{1},\ldots,v_{m}} p(v_{1},\ldots,v_{m}, y \mid \dop(X=x)) \\
      &= \sum_{x}p^{\star}(x\mid t)p(y\mid \dop(X=x)),
  \end{split}
\end{equation}
 where $V_{1},\ldots,V_{m}$ are all the variables in $\mathbf{V}$ except for $X$ and $Y$.

\subsection{Proof of the Backdoor Adjustment Formula for Variable Abstractions}
\label{sec:caus-abs-results}

\begin{proof}[Proof of \Cref{prop:backdoor-for-abs}] %
  \label{proof:backdoor}
  \begin{equation}
    \begin{split}
      p(y \mid \dop(t)) &= \sum_{x}p^{\star}(x\mid t) p^{\C;\dop(X=x)}(y) \\
                  &= \sum_{x}p^{\star}(x\mid t) \sum_{z} p^{\C}(z) p^{\C}(y\mid x,z) \\
                  &= \sum_{z}p^{\C}(z) \sum_{z} p^{\C}(y\mid x,z)  \frac{q(t\mid x) p^{\star}(x)}{\sum_{\dot{x}}q(t\mid \dot{x}) p^{\star}(\dot{x})},
      \end{split}
  \end{equation}
  where the second equality follows from the backdoor adjustment formula \citep{pearl2009causality}, and the last one from the definition of the intervention decoder.
\end{proof}

\section{Supplementary Material on Comparing Variable Abstractions}
\label{sec:appendix-comparing-abst}

\begin{proposition} %
  \label{prop:cong-is-equiv-rel}
  The relation $\cong$ of equivalence of abstractions is an equivalence relation.
\end{proposition}
\begin{proof}
  Reflexivity of $\cong$ is immediate: just take $\sigma = \mathrm{id}$.
  We now show that $\cong$ is symmetric.
  Assume that $T_{1} \cong T_{2}$, with corresponding bijection $\sigma$.
  Denote by $\sigma^{-1}$ its inverse.
  Let $t_{2} \in \supp(T_{2})$ and $t_{1} = \sigma^{-1}(t_{2})$.
  Then,
  \begin{equation}
    \begin{split}
      q_{T_{2} \mid X}(t_{2} \mid x) &= q_{T_{2} \mid X}(\sigma(t_{1}) \mid x) \\
        &= q_{T_{1} \mid X}(t_{1} \mid x)  \\
        &= q_{T_{1} \mid X}(\sigma^{-1}(t_{2}) \mid x).
    \end{split}
  \end{equation}
  This shows that $T_{2} \cong T_{1}$.
  Finally, we show transitivity.
  Let $T_{1}$, $T_{2}$ and $T_{3}$ be v-abstractions of $X$ such that $T_{1}\cong T_{2}$ and $T_{2}\cong T_{3}$, with bijections $\sigma_{12}$ and $\sigma_{23}$.
  Then,
  \begin{equation}
    \begin{split}
      q_{T_{1} \mid X}(t_{1} \mid x) &= q_{T_{2} \mid X}(\sigma_{12}(t_{1}) \mid x) \\
        &= q_{T_{3} \mid X}(\sigma_{23}(\sigma_{12}(t_{1})) \mid x).
    \end{split}
  \end{equation}
  Hence $T_{1} \cong T_{3}$ with bijection $\sigma_{23} \circ \sigma_{12}$.
\end{proof}

\begin{proposition}
  \label{prop:vi-zero-means-equiv}
  Let $T_{1}$ and $T_{2}$ be v-abstractions of $X$.
  If $\VI(T_{1}, T_{2}) = 0$, then $T_{1} \cong T_{2}$.
  Furthermore, the converse also holds if $T_{1}$ is a deterministic v-abstraction of $X$.
\end{proposition}
\begin{proof} %
  Note that $\VI(T_{1}, T_{2}) = 0$ if and only if $H(T_{2}\mid T_{1}) = H(T_{1}\mid T_{2}) = 0$.
  Recall that
  $H(T_{2}\mid T_{1}) = -\sum_{t_{1}\in\supp(T_{1})} p(t_{1})\sum_{t_{2}\in \supp(p_{T_{2}\mid T_{1}})}p(t_{2}\mid t_{1})\log p(t_{2}\mid t_{1})$,
  which is zero if and only if $p(t_{2}\mid t_{1}) \in \{0,1\}$ for all $t_{1}, t_{2}$ in the respective supports.
  By the same token, $p(t_{1}\mid t_{2}) \in \{0,1\}$.
  Define $\sigma\colon \supp(T_{1}) \to \supp(T_{2})$ by $\sigma(t_{1}) = \arg_{t_{2}}(p(t_{2} \mid t_{1}) = 1)$.
  Similarly, define $\sigma^{-1}\colon \supp(T_{2}) \to \supp(T_{1})$ by $\sigma^{-1}(t_{2}) = \arg_{t_{1}}(p(t_{1} \mid t_{2}) = 1)$.
  Then, $\sigma^{-1}$ is the inverse of $\sigma$.
  To see this, note that, for all $t_{2}^{*} \in \supp(T_{2})$, one has $\sigma(\sigma^{-1}(t_{2}^{*})) = \arg_{t_{2}}(p(t_{2}\mid \sigma^{-1}(t_{2}^{*})) = 1)$.
  But $p(t^{*}_{2}\mid \sigma^{-1}(t_{2}^{*})) =
    \frac{p(\sigma^{-1}(t_{2}^{*}) \mid t^{*}_{2}) p(t^{*}_{2})}{p(\sigma^{-1}(t_{2}^{*}))} =
    \frac{1 \cdot p(t^{*}_{2})}{p(\sigma^{-1}(t_{2}^{*}))}$, which cannot be zero.
    Since $p(t_{2}\mid t_{1}) \in\{0,1\}$ for all $t_{1}, t_{2}$, one then has that $p(t^{*}_{2}\mid \sigma^{-1}(t_{2}^{*})) = 1$, and therefore $\sigma(\sigma^{-1}(t_{2}^{*})) = t_{2}^{*}$.
  Now, let $t_{1} \in \supp(T_{1})$ and $t_{2} = \sigma(t_{1})$.
  Then $q_{T_{1}\mid X}(t_{1}\mid x) = q_{T_{1}\mid X}(\sigma^{-1}(t_{2})\mid x) = q_{T_{2}\mid X}(t_{2}\mid x) = q_{T_{2}\mid X}(\sigma(t_{1}) \mid x)$.
  Hence $T_{1} \cong T_{2}$. \\
  Assume now that $T_{1}$ is a deterministic v-abstraction of $X$, \emph{i.e.}, $q_{T_{1}\mid X}(t_{1}\mid x) \in \{0,1\}$ for all $t_{1}\in R_{T}, x\in \supp(X)$.
  Assume further that $T_{1}\cong T_{2}$ with bijection $\sigma$.
  Then clearly $T_{2}$ is also deterministic.
  Making use of the definition of v-abstraction, we have $p(t_{2} \mid t_{1}) = \sum_{x}  p(t_{2}\mid x) p(x\mid t_{1})$.
  If there is no $x$ for which $p(t_{2}\mid x) = 1$, then this is zero.
  Otherwise, $p(t_{2} \mid t_{1}) = \sum_{x}  p(t_{2}\mid x) \frac{p(t_{1}\mid x)p(x)}{\sum_{\dot{x}}p(t_{1}\mid \dot{x})p(\dot{x})} = \sum_{x}q_{T_{2}\mid X}(t_{2}\mid x) \frac{q_{T_{1}\mid X}(t_{1}\mid x)p(x)}{\sum_{\dot{x}}q_{T_{1}\mid X}(t_{1}\mid \dot{x})p(\dot{x})} = \sum_{x \in S} \frac{q_{T_{1}\mid X}(t_{1}\mid x)p(x)}{\sum_{\dot{x}}q_{T_{1}\mid X}(t_{1}\mid \dot{x})p(\dot{x})} = \sum_{x \in S} \frac{q_{T_{2}\mid X}(\sigma(t_{1})\mid x)p(x)}{\sum_{\dot{x}}q_{T_{2}\mid X}(\sigma(t_{1})\mid \dot{x})p(\dot{x})}$,
  where $S$ is the set of values $x$ where $q_{T_{2}\mid X}(t_{2}\mid x) = 1$.
  Notice that, for $x \in S$, we have that $q_{T_{2}\mid X}(\sigma(t_{1})\mid x)$ is zero if $\sigma(t_{1}) \ne t_{2}$ (and is one otherwise).
  Therefore, $p(t_{2}\mid t_{1}) = 0$ if $\sigma(t_{1}) \ne t_{2}$, and in case $\sigma(t_{1}) = t_{2}$ we have
  $p(t_{2} \mid t_{1}) =  \frac{\sum_{x \in S} q_{T_{2}\mid X}(\sigma(t_{1})\mid x)p(x)}{\sum_{\dot{x} \in S}q_{T_{2}\mid X}(\sigma(t_{1})\mid \dot{x})p(\dot{x})} = 1$.
  Hence $p(t_{2} \mid t_{1})$ can only take the values $0$ and $1$.
  From the expression for $H(T_{2}\mid T_{1})$ above, one concludes that $H(T_{2} \mid T_{1}) = 0$.
  A similar argument holds for $H(T_{1} \mid T_{2})$.
  This shows that $\VI(T_{1}, T_{2}) = 0$.
\end{proof}

\section{The Learning Algorithms}

A natural way for us to construct a method for minimizing \Cref{eq:cib_lag} would be to replicate the procedure in \citet{tishby2000information} using an implicit analytical solution for $\nabla_{q_{T\mid X}}\LIB = 0$ to formulate the minimization of $\LIB$ as a multiple-minimization problem solvable by a coordinate descent algorithm (see \Cref{sec:preliminaries}).
However, the expression resulting from $\nabla_{q_{T\mid X}}\LCIB = 0$ is much more complicated than that for $\nabla_{q_{T\mid X}}\LIB = 0$.
Furthermore, the derivation in \citet{tishby2000information} relied on the fact that mutual information can be written as a KL divergence, but the causal information gain cannot  \citep{simoes2024fundamental}.
So, although we did not show it is impossible to find a coordinate descent algorithm of the style of the Blahut-Arimoto algorithm, it is clear that a derivation of such an algorithm would have to take a very different form from the one in \citet{tishby2000information}.
Instead, we opted to find the minima of $\LCIB$ while staying constrained to the probability simplices using two types of projected gradient descent algorithms, which we now discuss.

\label{sec:learning_algo}
Since $T$ is discrete, each conditional distribution $q_{T\mid X=x}$ can be seen as a vector of probabilities $(q_{T=t \mid X=x})_{t\in R_{T}}$ in the Euclidean space $\R^{\abs{R_{T}}}$.
Furthermore, since $X$ is also discrete, the encoder $q_{T \mid X}$ can be seen as a vector of probabilities $$(q_{T = t_{1}\mid X=x_{1}},\ldots,q_{T = t_{1}\mid X=x_{\abs{R_{X}}}}, q_{T = t_{2}\mid X=x_{1}}, \ldots, q_{T = t_{\abs{R_{T}}}\mid X=x_{\abs{R_{X}}}}) \in \R^{\abs{R_{T}} \cdot \abs{R_{X}}}.$$
The $\abs{R_{T}}$-dimensional subspace of $\R^{\abs{R_{T}} \cdot \abs{R_{X}}}$ corresponding to $q_{T\mid X=x}$ will be denoted $\mathbb{R}^{(\cdot\mid x)}$.
Each conditional distribution vector $(q_{T=t \mid X=x})_{t\in R_{T}}$ must lie in the probability simplex $\Delta^{\abs{R_{T}} - 1}$, so that $q_{T\mid X}$ must belong to the cartesian product of simplices $\Delta \coloneqq \bigtimes_{x \in R_{X}} \Delta^{\abs{R_{T}} - 1}$.
The causal information bottleneck problem can then be formulated as finding the global minimum of $\LCIB[q_{T\mid X}]$ with $q_{T\mid X}$ constrained to $\Delta$.
We will use projected Gradient Descent (pGD) and projected simulated annealing gradient descent (pSAGD).
The pSAGD algorithm takes its inspiration from simulated annealing methods which help hill-climbing algorithms avoid local minima in discrete search spaces \citep{norvig2021artificial}.
The pSAGD is a simulated annealing version of pGD which does something analogous.
Being itself a local search algorithm, pSAGD still may converge to local minima.
Re-running pSAGD a few times (effectively using an ensemble of pSAGD learners) will increase the likelihood of finding the global minimum.

\paragraph{Projected Gradient Descent}
In projected gradient descent (pGD) \citep{bertsekas2016nonlinear}, each iteration step is of the form:
\begin{equation}
  q_{T\mid X}^{(t+1)} = \Pi_{\Delta}\left( q_{T\mid X}^{(t)} - \alpha \nabla_{q_{T\mid X}}(\LCIB)\raisebox{-0.2em}{\Big\vert}_{q_{T\mid X}^{(t)}} \right),
\end{equation}
where $\Pi_{\Delta}$ is the Euclidean projection onto the constraint space $\Delta$ and $\alpha$ is the learning rate for the gradient descent step.
It is easy to check that projecting onto $\Delta$ is equivalent to projecting each $q_{T\mid X=x}$ onto the probability simplex $\Delta^{\abs{R_{T}}-1}$.
In other words, we can apply the projection to each conditional distribution $q_{T\mid X=x} \in \mathbb{R}^{\abs{R_{T}}}$ separately.
That is,
\begin{equation}
  q_{T\mid X=x}^{(t+1)} = \Pi_{\Delta^{\abs{R_{T}} - 1}}\left( q_{T\mid X=x}^{(t)} - \alpha \nabla_{q_{T\mid X=x}}(\LCIB\big\vert_{\mathbb{R}^{(\cdot\mid x)}})\raisebox{-0.2em}{\Big\vert}_{q_{T\mid X=x}^{(t)}} \right).
\end{equation}
This is illustrated in \Cref{fig:simplex-proj}.
The Euclidean projection onto each probability simplex is done using our implementation of the algorithm for projection onto the probability simplex described in \citet{duchi2008efficient}.
Our implementation is vectorized, so that we can simultaneously apply the projection to each $q_{T\mid X=x}$ separately.

\paragraph{Projected Simulated Annealing Gradient Descent}
In complex systems, multiple local minima of the CIB may exist.
Projected simulated annealing gradient descent (pSAGD) addresses this issue by introducing randomness, or ``jittering'', into the gradient descent process, controlled by a temperature parameter.
At each iteration $t$, instead of taking a step in the direction of the negative gradient (followed by a projection) at the current point $q^{(t)}$, there is a chance of jumping to a random neighbor, that is, a point $\tilde{q} \in \Delta$ obtained by uniformly sampled from the sphere centered on $q^{(t)}$ of radius equal to the learning rate (and applying a projection onto the simplex if necessary).
The probability of accepting the proposed $\tilde{q}$ depends on the temperature and the quality of the proposal.
Specifically, if the loss $L = \LCIB$ at $\tilde{q}$ is lower than $L(q^{(t)})$, then $\tilde{q}$ is accepted, \emph{i.e.}, $q^{(t+1)}=\tilde{q}$.
Otherwise, the acceptance probability is given by $\exp(-\frac{L(\tilde{q}) - L(q^{(t)})}{T^{(t)}})$, where $T^{(t)}$ is the current temperature.
The temperature decreases according to a chosen cooling rate $c\in (0, 1)$, such that $T^{(t+1)}=c\cdot T^{(t)}$.
If $\tilde{q}$ is not accepted, a pGD step is taken instead.

\begin{figure}[h!]
  \centering
  \begin{tikzpicture}[scale=3.5]
    \draw[thick,fill=olive!40!green!50,opacity=0.5] (0.5,0) -- (0,0.5) -- (-0.25,-0.25) -- cycle;
    \node[color=olive!40!green] (sp) at (-0.20, 0.25) {$\Delta^{2}$};

    \node[below] (n1) at (0.5, 0) {\tiny$(1,0,0)$};
    \node[xshift=13pt, yshift=5pt] (n2) at (0, 0.5) {\tiny$(0,1,0)$};
    \node[xshift=10pt, yshift=-5pt] (n3) at (-0.25, -0.25) {\tiny$(0,0,1)$};
    \node[circle, fill, inner sep=1pt] (qpre) at (0.5, 0.4) {};
    \node[right] at (qpre) {\small$q'$};
    \node[circle, fill, inner sep=1pt] (qt) at (0.08, 0.22) {};
    \node[above] at (qt) {\small$q^{(t)}$};
    \node[circle, fill, inner sep=1pt] (q) at (0.20, 0.10) {};
    \node[right] at (q) {\small$q^{(t+1)}$};

    \draw[-latex] (0,0) -- (1,0) node[right] {\small$q_{T=t_{1}, X=x}$};
    \draw[-latex] (0,0) -- (0,1) node[above] {\small$q_{T=t_{2}, X=x}$};
    \draw[-latex] (0,0) -- (-0.5,-0.5) node[below] {\small$q_{T=t_{3}, X=x}$};
    \draw[-latex, dashed, color=red] (qpre) -- node[pos=0.1,below,inner sep=9pt] {\tiny$\Pi_{\Delta^{2}}$} (q);
    \draw[-latex, dashed, color=blue] (qt) -- node[pos=0.5,above,inner sep=2pt] {\tiny GD} (qpre);
    \draw[-latex, dashed, color=orange] (qt) -- node[pos=0.55,below left,inner sep=1pt] {\tiny pGD} (q);
  \end{tikzpicture}
  \caption{Illustration of a Projected Gradient Descent step in the $\abs{R_{T}}$-dimensional slice $\mathbb{R}^{(\cdot\mid x)} \cong \mathbb{R}^{\abs{R_{T}}}$ of $\mathbb{R}^{\abs{R_{T}}\cdot \abs{R_{X}}}$ corresponding to the conditional distribution $q_{T\mid X=x}$, for the case where $\abs{R_{T}}=3$.
  Here, $q^{(t)}$ denotes $q^{(t)}_{T\mid X=x}$ and $q'$ denotes the output of the Gradient Descent (GD) step, that is, $q' = q_{T\mid X=x}^{(t)} - \alpha \nabla_{q_{T\mid X=x}}(\LCIB\big\vert_{\mathbb{R}^{(\cdot\mid x)}})\raisebox{-0.2em}{\Big\vert}_{q_{T\mid X=x}^{(t)}}$.}
  \label{fig:simplex-proj}
\end{figure}

\paragraph{Implementation details}

We employed gradient clipping to prevent gradient explosions, which can destabilize the optimization process.
Additionally, a cycle detection mechanism was incorporated, so that the optimization process stops if the algorithm cycles between two encoders, at which point the best solution is selected.
We observed that using relatively large learning rates was necessary for effective training.
This appears to stem from the small gradients encountered during optimization.
Unsurprisingly, we also observed that, the higher the dimension of the search space, the larger the learning rates need to be.
A cooling rate of $0.99$ was selected for the simulated annealing schedule, as it provides a reasonable decay curve for the probability of acceptance across different temperatures and typical step sizes, balancing the exploration and exploitation phases effectively.
During the experiments, we noted that many local minima at which the optimizers got stuck corresponded to encoders that did not fully utilize the entire range of $T$.
We call such v-abstractions non-surjective.
We want to avoid such v-abstractions, since we would like the chosen range $R_{T}$ to be respected.
To address this, we introduced a penalty term in the loss function that discouraged the optimizer from approaching non-surjective encoders.
We tested this for $\gamma=1$, where this adjustment greatly improved the accuracy of the optimizer, that is, the frequency at which it identified the ground truth.
Despite achieving these high accuracies, we employed ensembles of optimizers to further enhance the likelihood of finding the global minimum.
All experiments, along with the exploration of different parameter settings and additional results, are available in the accompanying code repository.

\section{Supplement to the Experimental Results}
\label{sec:appendix-exp-details}

\begin{figure}[h]
  \centering
  \scalebox{1.0}{
  \begin{tikzpicture}[mynode/.style={circle,draw=black,fill=white,inner sep=0pt,minimum size=0.8cm}]
  \node[mynode] (x) at (0,0) { $X$};
  \node[mynode] (y) at (2,0) { $Y$};
  \path [draw,->] (x) edge[-latex] (y);
  \node (assigns) at (1.4,-1.4)
    {
      $
      \begin{cases}
        Y \vcentcolon= X + N_{Y} \mod 2 \\
        N_{X} \sim U\{0,\ldots,6\} \\
        N_{Y} \sim \mathrm{Bern}(u_{Y}), u_{Y} < 0.5
      \end{cases}
      $
    };
  \end{tikzpicture}
  } %
  \caption{SCM for the Odd and Even experiment.}%
  \label{fig:odd-and-even}
\end{figure}

\subsection{Other Noise Distributions}
\label{sec:appendix-other-noises}
We tested our method on different parameterizations of the noise distributions, for each model from \Cref{sec:experimental-results}.
The results are presented in parallel coordinate plots (\Cref{fig:exps-other-noises}).
Notice the similarity to the findings reported in the main text (\Cref{fig:exp-results-fixed-param}).
The analysis in \Cref{sec:experimental-results} remains valid for the results in \Cref{fig:exps-other-noises}.
In particular, conditions (a), (b), and (c) are still satisfied when using these other noise distributions.
Interactive HTML versions for these plots are available in the code repository.

\begin{figure}[h]
  \centering
  \begin{subfigure}[t]{0.32\textwidth}
    \centering
    \includegraphics[trim=50 0 0 0, clip, scale=\imagescale]{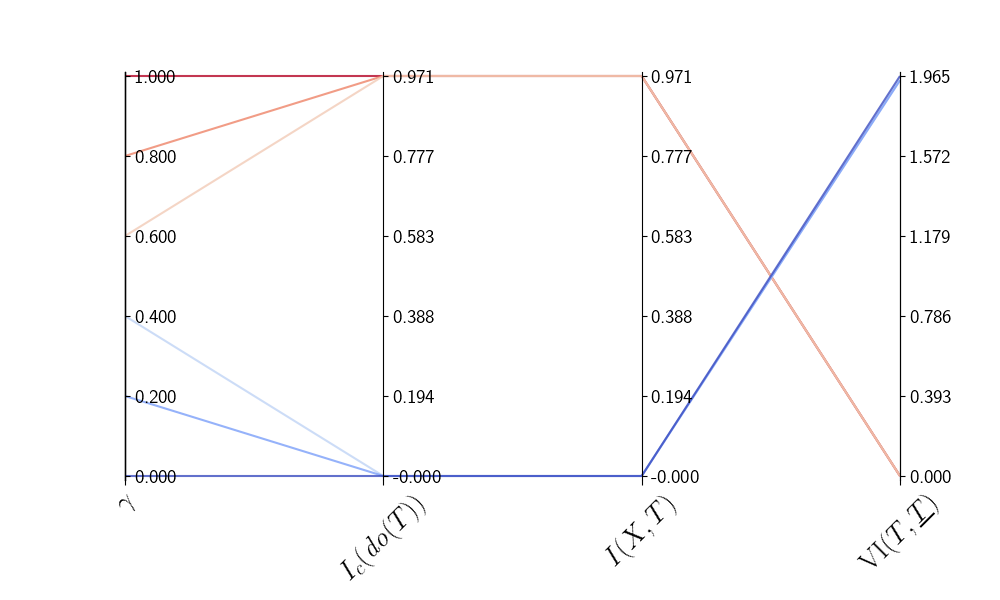}
    \caption{$u_{Y} = 0.0$}
  \end{subfigure}
  \hfill
  \begin{subfigure}[t]{0.32\textwidth}
    \centering
    \includegraphics[trim=50 0 0 0, clip, scale=\imagescale]{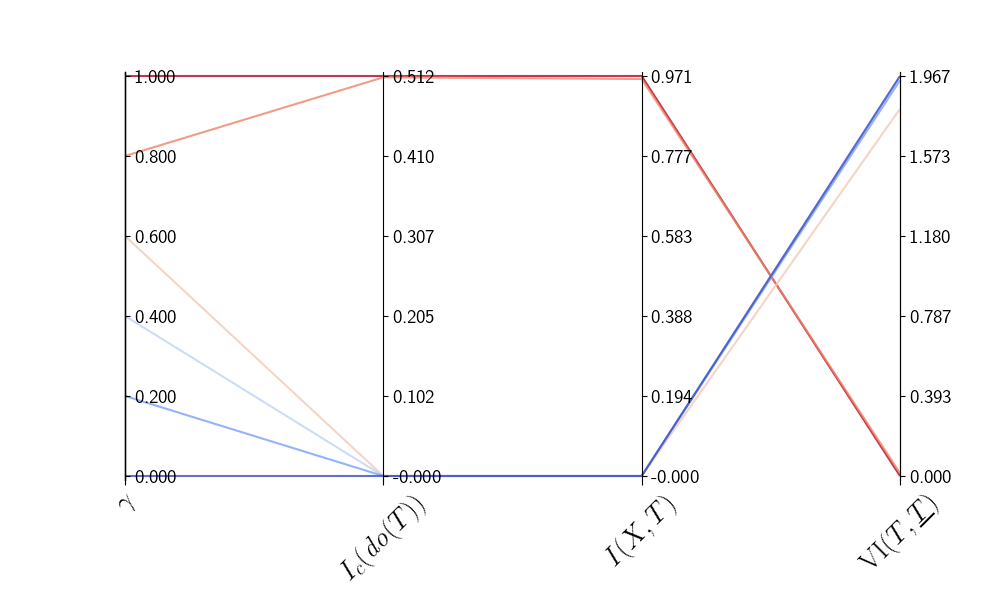}
    \caption{$u_{Y} = 0.1$}
  \end{subfigure}
  \hfill
  \begin{subfigure}[t]{0.32\textwidth}
    \centering
    \includegraphics[trim=50 0 0 0, clip, scale=\imagescale]{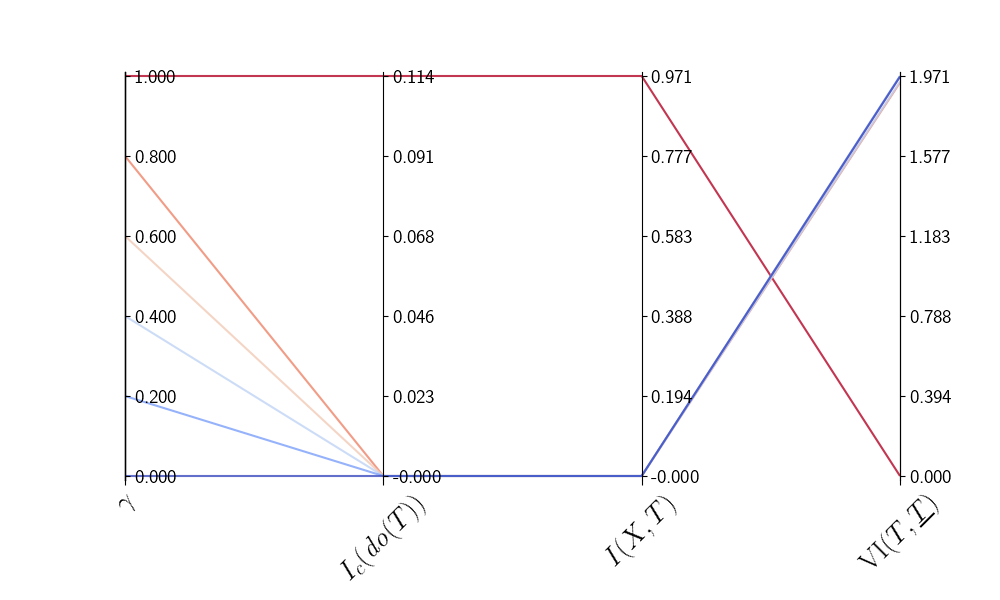}
    \caption{$u_{Y} = 0.3$}
  \end{subfigure}
  \\
  \hfill
  \begin{subfigure}[t]{0.32\textwidth}
    \centering
    \includegraphics[trim=50 0 0 0, clip, scale=\imagescale]{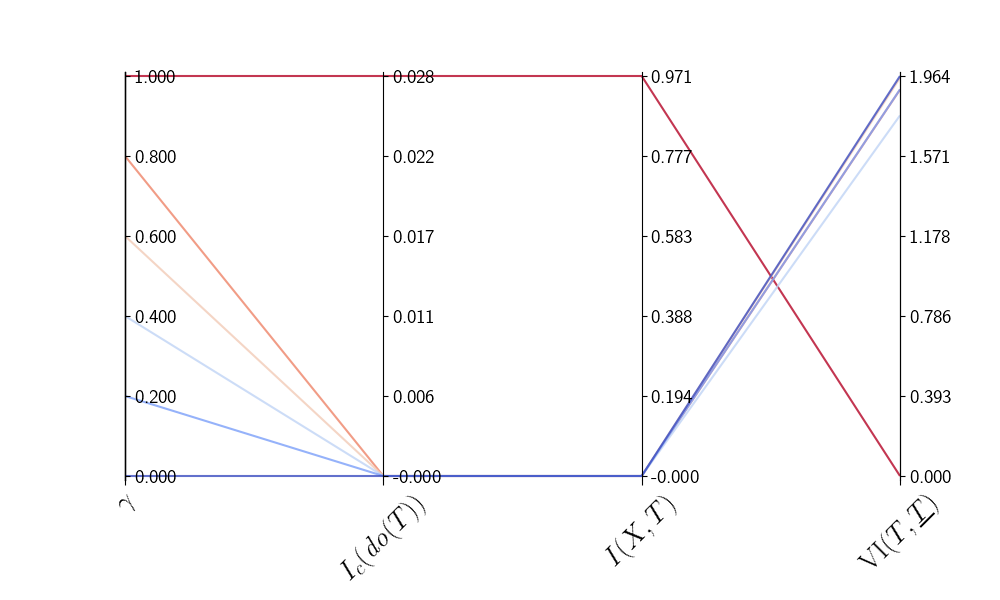}
    \caption{$u_{Y} = 0.4$}
   \end{subfigure}
  \hfill
  \begin{subfigure}[t]{0.32\textwidth}
    \centering
    \includegraphics[trim=50 0 0 0, clip, scale=\imagescale]{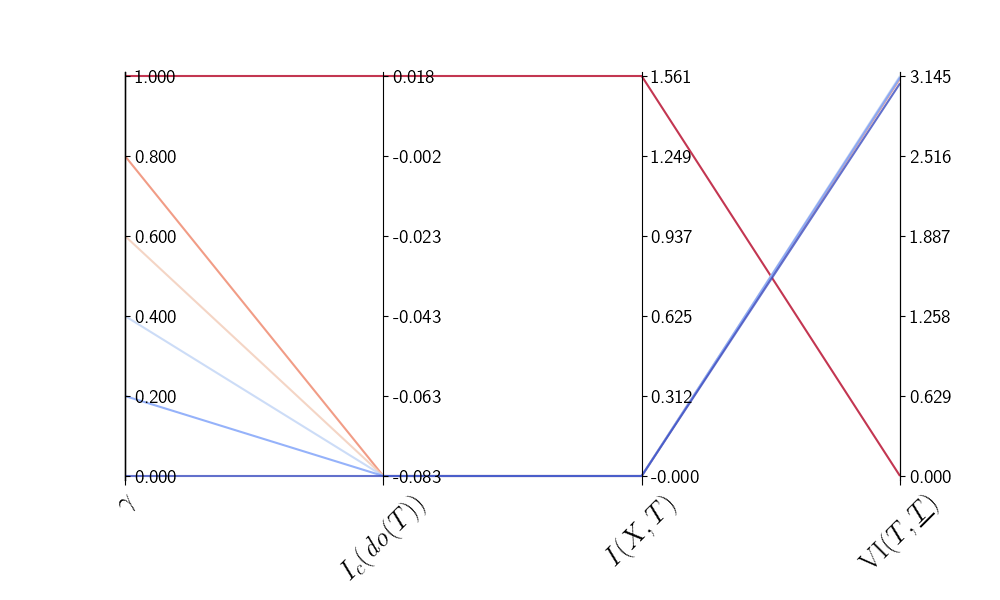}
    \caption{$r_{Y} = 0.1$}
  \end{subfigure}
  \hfill
  \begin{subfigure}[t]{0.32\textwidth}
    \centering
    \includegraphics[trim=50 0 0 0, clip, scale=\imagescale]{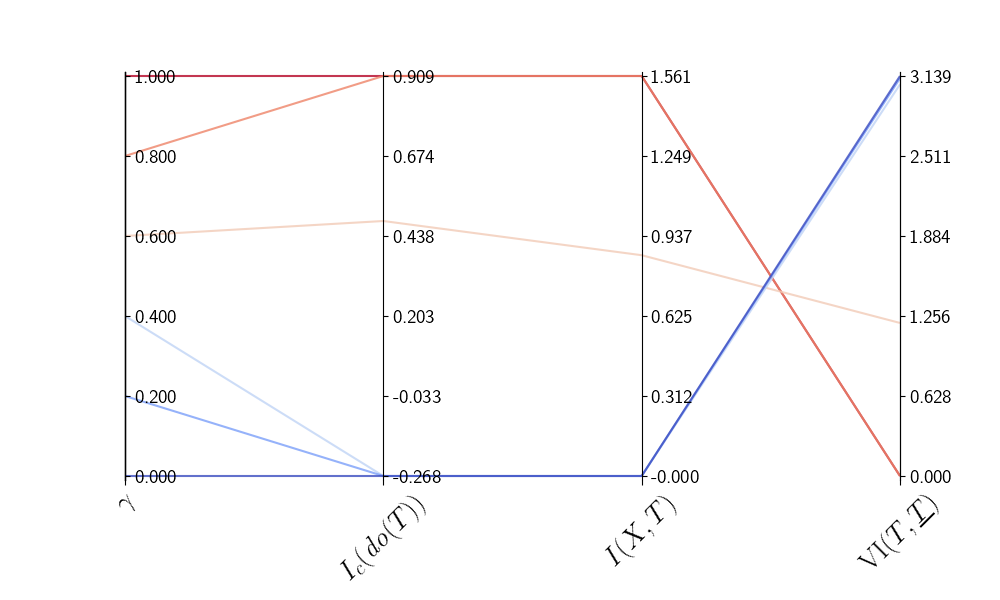}
    \caption{$r_{Y} = 0.9$}
  \end{subfigure}
  \\
  \hfill
  \begin{subfigure}[t]{0.32\textwidth}
    \centering
    \includegraphics[trim=50 0 0 0, clip, scale=\imagescale]{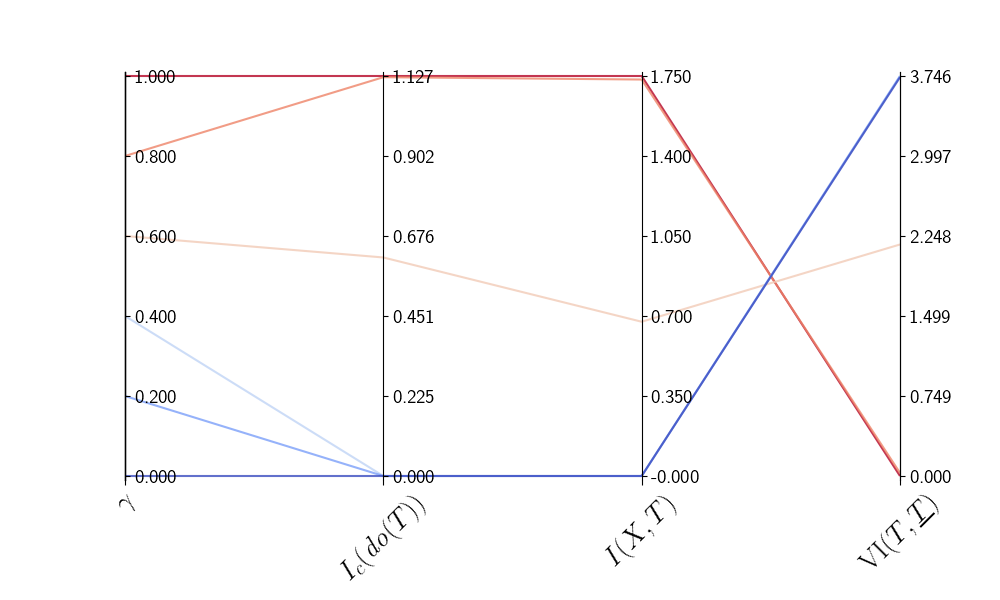}
    \caption{$b_{X_{i}} = 0.5$, $b_{Y} = 0.1$ and $b_{S} = 0.5$}
   \end{subfigure}
  \hfill
  \begin{subfigure}[t]{0.32\textwidth}
    \centering
    \includegraphics[trim=50 0 0 0, clip, scale=\imagescale]{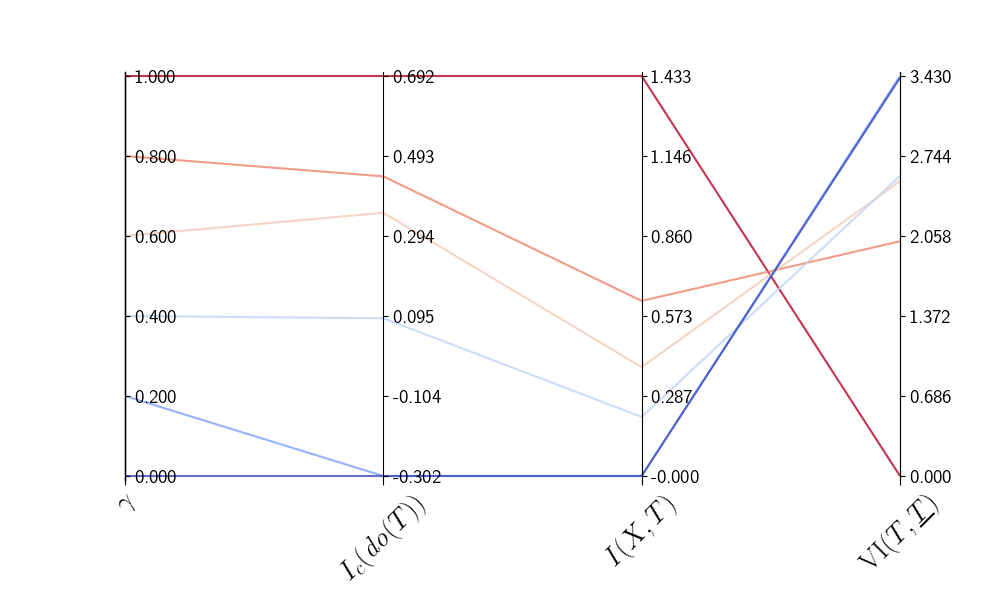}
    \caption{$b_{X_{i}} = 0.3$, $b_{Y} = 0.4$ and $b_{S} = 0.5$}
  \end{subfigure}
  \hfill
  \begin{subfigure}[t]{0.32\textwidth}
    \centering
    \includegraphics[trim=50 0 0 0, clip, scale=\imagescale]{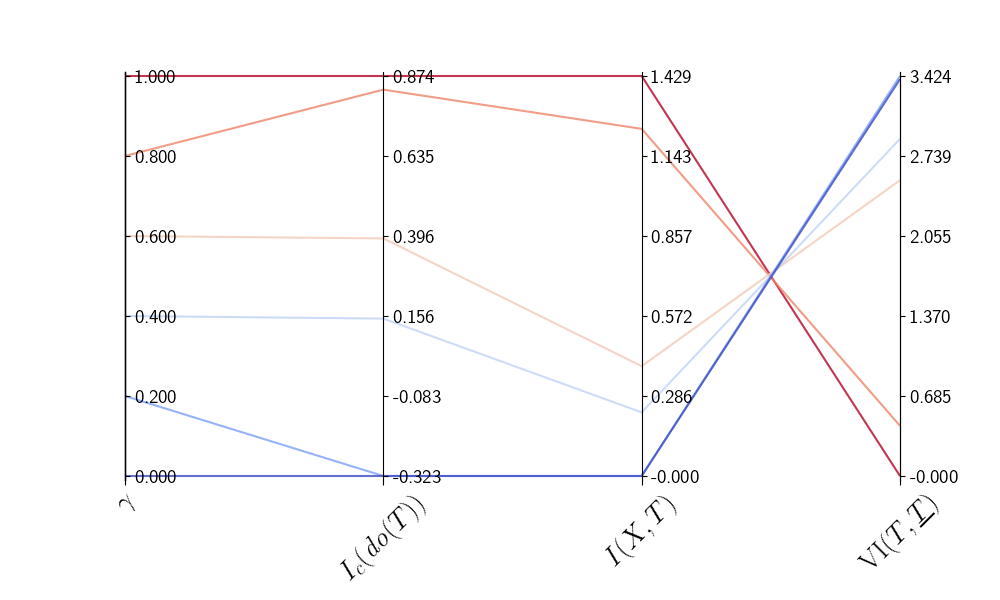}
    \caption{$b_{X_{i}} = 0.3$, $b_{Y} = 0.1$ and $b_{S} = 0.8$}
  \end{subfigure}
  \caption{Experimental results for other noise distributions. Each line corresponds to the v-abstraction found for the chosen $\gamma \in \{0, 0.2, \ldots, 1.0\}$. Figures (a)-(d) refer to the Odd and Even experiment; (e) and (f) to the Confounded Addition experiment; and (g)-(i) to the Genetic Mutations experiment.}
  \label{fig:exps-other-noises}
\end{figure}

\subsection{Comparison with the Standard Information Bottleneck Method}
\label{sec:appendix-IB}

We ran the Information Bottleneck (IB) method on the three experiments, in order to compare the results with those that we had obtained using the CIB method, and confirm experimentally that the IB method fails in our task of learning OCVAs.
The full results are presented in \Cref{fig:ib_tests}.
The obvious difference in comparison to \Cref{fig:exp-results-fixed-param} is that the ground truth is never found, even when $\gamma=1$, in the Confounded Addition and Mutations experiments.
This was to be expected: the v-abstractions learned with the IB method do not successfully capture the aspects of $X$ which causally affect $Y$ whenever there is confounding.

Since in the Odd and Even experiment there is no confounding, the effect of interventions on $X$ is indistinguishable from the effect of conditioning on $X$, so that we expect the v-abstraction with the most causal control and the v-abstraction with the most predictive power over $Y$ to coincide. This is indeed what happens.

Let us now take a closer look at the $\gamma=1$ case, for the experiments with confounding.

\begin{figure}[h]
  \centering
  \begin{subfigure}[t]{0.32\textwidth}
    \centering
    \includegraphics[trim=50 0 0 0, clip, scale=\imagescale]{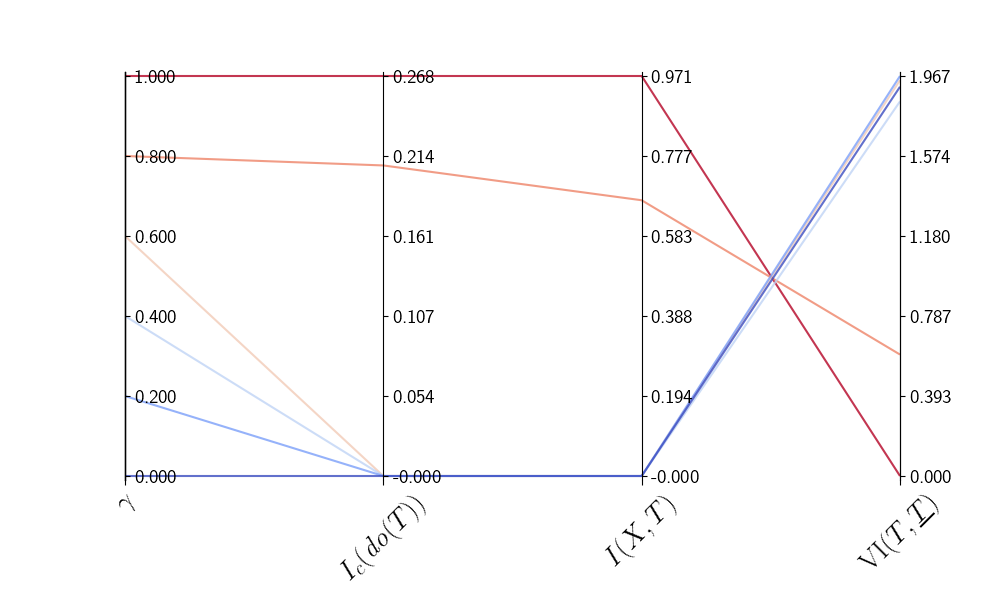}
    \caption{Odd and Even experiment}
  \end{subfigure}
  \hfill
  \begin{subfigure}[t]{0.32\textwidth}
    \centering
    \includegraphics[trim=50 0 0 0, clip, scale=\imagescale]{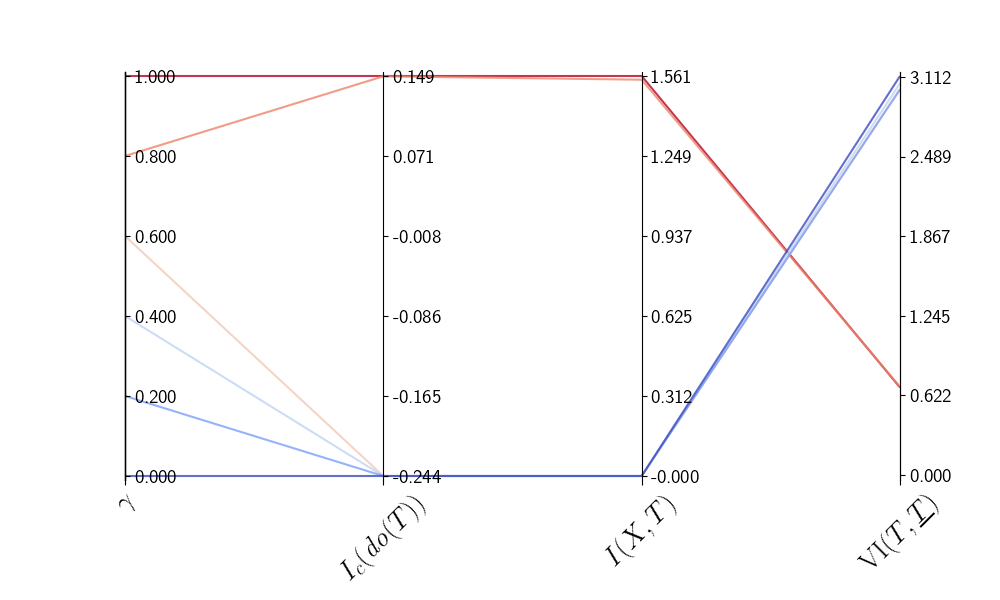}
    \caption{Confounded Addition experiment}
  \end{subfigure}
  \hfill
  \begin{subfigure}[t]{0.32\textwidth}
    \centering
    \includegraphics[trim=50 0 0 0, clip, scale=\imagescale]{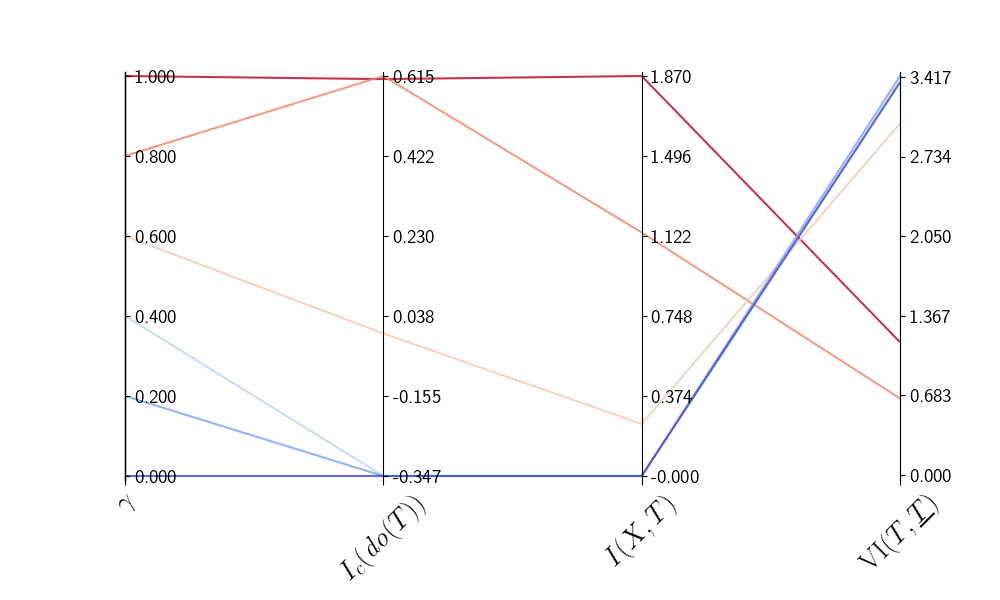}
    \caption{Mutations Experiment}
  \end{subfigure}
  \caption{Results from running the standard Information Bottleneck (IB) method on the three experiments, using the same noise distributions as in \Cref{fig:exp-results-fixed-param}.
  The IB fails to find the ground truth when there is confounding.}
  \label{fig:ib_tests}
\end{figure}

The v-abstraction learned using the IB method (using $\gamma=1$) for the Confounded Addition experiment was the map presented in \Cref{table:ib-encoder-addition}.
Notice that this erroneously (from a causal perspective) maps the cases $(X_1 = 0, X_2 = 1)$ and $(X_1 = 1, X_2 = 0)$ to different values of $T$, even though these two cases result in exactly the same causal effect on $Y$ (but not the same (non-causal) predictive power over $Y$, due to the confounding through $W$).

\begin{table}[h]
\centering
\begin{tabular}{cc|c}
\toprule
$X_1$ & $X_2$ & $T$ \\
\midrule
0 & 0 & 0 \\
0 & 1 & 0 \\
1 & 0 & 1 \\
1 & 1 & 2 \\
\bottomrule
\end{tabular}
\caption{Encoder learned by the Information Bottleneck method (using $\gamma=1$), for the Confounded Addition experiment.}
\label{table:ib-encoder-addition}
\end{table}

For the Genetic Mutations experiment, the v-abstraction learned using the IB method (using $\gamma=1$) was the map presented in \Cref{table:ib_mutations_encoder}
Notice that, in contrast with the ground truth $\underline{T}$ (which is successfully learned by the CIB), this v-abstraction fails to distinguish the cases $(X_1 = 1, X_2 = 1, X_3 = 0, X_4 = x_4)$ and $(X_1 = 1, X_2 = 1, X_3 = 1, X_4 = x_4)$; that is, it fails to capture the protective effect of $X_3$ discussed in \Cref{sec:experimental-results}.
Hence, the IB method fails to learn the complex, epistatic interactions between the mutations that the CIB successfully learns.
Notice further that the v-abstraction learned by the IB is not invariant on $X_4$. This is due to the confounding between $X_4$ and $Y$, which the CIB, in contrast, successfully handles.

\begin{table}[h]
\centering
\begin{tabular}{cccc|c}
\toprule
$X_1$ & $X_2$ & $X_3$ & $X_4$ & $T$ \\
\midrule
0 & 0 & 0 & 0 & 1 \\
0 & 0 & 1 & 0 & 1 \\
0 & 1 & 0 & 0 & 3 \\
0 & 1 & 1 & 0 & 3 \\
1 & 0 & 0 & 0 & 3 \\
1 & 0 & 1 & 0 & 3 \\
1 & 1 & 0 & 0 & 2 \\
1 & 1 & 1 & 0 & 2 \\
0 & 0 & 0 & 1 & 1 \\
0 & 0 & 1 & 1 & 0 \\
0 & 1 & 0 & 1 & 3 \\
0 & 1 & 1 & 1 & 3 \\
1 & 0 & 0 & 1 & 3 \\
1 & 0 & 1 & 1 & 3 \\
1 & 1 & 0 & 1 & 2 \\
1 & 1 & 1 & 1 & 2 \\
\bottomrule
\end{tabular}
\caption{Encoder learned by the Information Bottleneck method, for the Genetic Mutations experiment.}
\label{table:ib_mutations_encoder}
\end{table}

\end{document}